\documentclass[11pt]{article}
\usepackage{enumerate}
\usepackage{pdfsync}
\usepackage[OT1]{fontenc}

\usepackage[usenames]{color}
\usepackage{smile}
\usepackage[colorlinks,
linkcolor=red,
anchorcolor=blue,
citecolor=blue
]{hyperref}
\usepackage{fullpage}
\usepackage[protrusion=true,expansion=true]{microtype}
\usepackage{pbox}
\usepackage{setspace}
\usepackage{tabularx}
\usepackage{float}
\usepackage{wrapfig,lipsum}
\usepackage{enumitem}
\usepackage{microtype}
\usepackage{graphicx}
\usepackage{subfigure}
\usepackage{pgfplots}
\usepackage{mathtools}
\usepackage{caption}
\usetikzlibrary{arrows,shapes,snakes,automata,backgrounds,petri}
\usepackage{booktabs} 

\allowdisplaybreaks
\usepackage{colortbl}

\usepackage{xcolor}

\usepackage{makecell}
\usepackage{pifont}

\makeatletter
\newcommand*{\rom}[1]{\expandafter\@slowromancap\romannumeral #1@}
\makeatother
\title{\huge Empirical Bayesian Multi-Bandit Learning}
\author
{
	Xia Jiang\thanks{Institute of Science and Technology for Brain-inspired Intelligence, Fudan University, Shanghai, People’s Republic of China; e-mail: {\tt xiajiang21@m.fudan.edu.cn}}
	~~~~
	Rong J.B. Zhu\thanks{PolicyLab, Children’s Hospital of Philadelphia, Philadelphia, US; e-mail: {\tt zhur1@chop.edu} (corr. author)}
}

\definecolor{LightCyan}{rgb}{0.58, 0.94, 0.85}
\ifdefined\final
\usepackage[disable]{todonotes}
\else
\usepackage[textsize=tiny]{todonotes}
\fi
\begin{document}
	\date{}
	\maketitle
	
	\begin{abstract}
		Multi-task learning in contextual bandits has attracted significant research interest due to its potential to enhance decision-making across multiple related tasks by leveraging shared structures while accommodating task-specific heterogeneity.  
		In this article, we propose a novel hierarchical Bayesian framework for learning across multiple bandit instances. This framework captures both heterogeneity and dependence among bandit instances through a hierarchical Bayesian model, enabling effective information sharing while accommodating instance-specific variations. 
		Unlike existing methods that overlook the learning of the covariance structure across bandits, we introduce an empirical Bayesian approach to estimate the covariance matrix of the prior distribution.
		This enhances the practicality and flexibility of learning across multi-bandits. 
		Building on this approach, we develop two efficient algorithms, \textsc{ebmTS} (empirical Bayesian multi-bandit Thompson Sampling) and \textsc{ebmUCB} (empirical Bayesian multi-bandit Upper Confidence Bound), which incorporate the estimated prior into the decision-making process. We establish the frequentist regret upper bounds for the proposed algorithms, providing theoretical guarantees for learning across multiple bandit instances. Extensive experiments on synthetic and real-world datasets demonstrate the superior performance of our proposed methods, particularly in complex environments. Our algorithms achieve lower cumulative regret compared to existing approaches across a range of settings, highlighting their ability to effectively balance exploration and exploitation while leveraging shared information across bandits.
	\end{abstract}
	
\section{Introduction}
\label{sec:introduction}

The contextual bandit problem provides a fundamental framework for analyzing decision-making in uncertain environments \citep{li2010contextual, agrawal2013thompson, abeille2017linear, bouneffouf2020survey, bastani2021mostly}. 
It involves a decision-maker who, at each time step, observes a context vector describing the current state of the environment before making a choice among available arms.
At each time step, the decision-maker must choose one of several available arms (or actions) based on the observed context. The chosen arm yields a reward, which is typically a function of both the context and the arm taken. The objective is to maximize cumulative reward over time by balancing exploration --- trying new arms to gather information --- and exploitation --- choosing arms that have historically yielded high rewards \citep{auer2002finite, bubeck2012regret, agrawal2017near}. The adaptability of contextual bandits to changing environments and their ability to optimize decisions based on contextual information make them a powerful tool for decision-making.
Contextual bandits have been widely studied and applied in various domains, including personalized recommendation systems \citep{li2010contextual, yang2020exploring, guo2020deep}, dynamic pricing \citep{misra2019dynamic, mueller2019low}, and healthcare \citep{woodroofe1979one, mate2020collapsing}. 

Despite the extensive research on contextual bandits, most existing work focuses on single-bandit scenarios \citep{li2010contextual, agrawal2013thompson}, where the decision-making process is confined to a single bandit instance. 
More recently, \citet{lee2023context} study context-lumpable contextual bandits, where contexts can be grouped according to a shared latent reward structure. 
However, learning across multiple bandit instances is common in practice. 
For example, in movie recommendation systems \citep{qin2014contextual}, there are numerous movies waiting to be recommended to different users. 
These users can be viewed as separate bandit instances, with each movie representing an arm. While certain movies may appeal similarly across users, the actual outcomes can vary significantly due to differences in users' preferences and characteristics.
In such settings, traditional models that assume shared parameters across arms fall short, as they fail to capture user-specific variations and the nuanced influence of each movie on individual outcomes.
Treating users as homogeneous overlooks the inherent heterogeneity among individuals. Conversely, modeling each user independently fails to leverage the similarities across users and leads to poor data usage, resulting in inefficient learning and suboptimal performance. This dichotomy highlights the necessity for a multi-task learning approach \citep{soare2014multi, deshmukh2017multi, wan2021metadata, fang2015active, su2024multi, hong2023multi} that can effectively capture both the similarities and differences among multiple bandit instances. By doing so, such an approach can enhance learning efficiency and improve decision-making in complex, heterogeneous environments.

In the realm of multi-task learning, various strategies have been proposed to address the challenges of learning across multiple tasks. One common approach builds on
meta-learning \citep{wan2021metadata, kveton2021meta}, where each task is associated with a task-specific parameter vector, and these parameters are typically drawn from a common distribution, allowing for the transfer of knowledge across tasks through a hierarchical structure \citep{hong2023multi, hong2022hierarchical}. 
However, applying this approach in the context of multi-armed learning is challenging, as the task-specific parameter is shared across all arms within a task, implying inherent correlations among the arms of the same task.
To address this issue,  prior work \citep{xu2025multitask, huang2025optimal, cella2021multi} imposed a sparse heterogeneity assumption on the arm parameters. While this assumption aims to capture both within-task heterogeneity and cross-task correlations among arms, it can result in poor performance when the assumption does not hold --- highlighting the limited generalizability of such methods.

In this paper, we introduce a hierarchical Bayesian model designed to capture both the diversity and the shared structure across multiple bandit instances. Each arm is modeled with its own parameter vector, reflecting the natural heterogeneity among arms, while these parameters are connected through a shared prior distribution that enables knowledge transfer across tasks. This shared prior allows the model to learn from related bandits, improving decision-making even when data from individual instances is limited.

The hierarchical framework assumes a Gaussian prior and updates parameter beliefs through Bayes' theorem. To make the model adaptive and data-driven, we estimate the covariance matrix of the prior and the noise variance using an empirical Bayesian approach, allowing the shared structure to emerge naturally from the data rather than being fixed in advance.

Building on this foundation, we design two exploration strategies --- Thompson Sampling (\textsc{TS}) and Upper Confidence Bound (\textsc{UCB}) --- that utilize the estimated posterior distributions. These lead to two algorithms: empirical Bayesian multi-bandit Thompson Sampling (\textsc{ebmTS}) and empirical Bayesian multi-bandit Upper Confidence Bound (\textsc{ebmUCB}). We further derive the frequentist regret bounds for these algorithms, which explicitly illustrate how the incorporation of the prior influences learning performance. Through extensive experiments on both synthetic and real-world datasets, we demonstrate that our methods consistently achieve superior performance and greater robustness relative to existing approaches.

In summary, our key contributions are as follows: 
\begin{itemize}
	\item We introduce a hierarchical Bayesian model that captures shared structure across multiple bandit instances. This framework accommodates both heterogeneity and correlations among bandit instances, enabling efficient information sharing while preserving instance-specific flexibility.
	\item We develop an empirical Bayesian approach to learn the shared structure across multiple bandit instances, with a key contribution being the estimation of the covariance matrix --- an essential yet often overlooked component. By employing the thresholded covariance matrix estimator \citep{bickel2008covariance}, our approach offers an automatic, data-driven way to uncover meaningful correlations across bandit instances. This method preserves strong correlations while eliminating weaker ones, effectively discarding those that do not contribute meaningfully to across-instance learning. We also propose computational techniques that enhance scalability and make our algorithms practical for large-scale applications. 
	\item Building on our empirical Bayesian framework, we develop two algorithms --- \textsc{ebmTS} and \textsc{ebmUCB} --- that integrate Thompson Sampling and UCB-based exploration strategies for multi-bandit problems. These algorithms are evaluated across diverse datasets, demonstrating strong effectiveness and robustness in various application settings. 
	\item We further establish a rigorous theoretical foundation by deriving an upper bound on the frequentist regret. Unlike most existing studies on multi-bandit problems \citep{hong2022hierarchical, aouali2023mixed}, which primarily analyze Bayesian regret, we explicitly characterize the estimation error and establish its upper bound with the incorporation of prior information --- a topic that has rarely been addressed in the multi-bandit literature. Building on this analysis, Theorem~\ref{inequality_beta_hat} provides the resulting frequentist regret bound, thereby filling an important gap in the multi-bandit literature.
\end{itemize}

The remainder of this article is organized as follows. Section~\ref{sec:related_works} provides a comprehensive review of the literature on contextual bandits and multi-task learning. Section~\ref{sec:problem_formulation} introduces the multi-bandit model and formulates the problem. Section~\ref{sec:estimation} details the estimation procedure. Section~\ref{sec:algorithm} presents the proposed \textsc{ebmTS} and \textsc{ebmUCB} algorithms, developed based on the estimation results. Section~\ref{sec:regret_analysis} provides the frequentist regret bounds for \textsc{ebmTS} and \textsc{ebmUCB}. Section~\ref{sec:experiment} reports experimental results on both synthetic and real-world datasets. Finally, Section~\ref{sec:conclusion} concludes the article.

\section{Related Works}
\label{sec:related_works}

Our work is closely related to \citet{xu2025multitask}, who propose a robust method for high-dimensional bandit problems under the assumption that the parameters of different bandits are sparse deviations from a shared global parameter. \citet{huang2025optimal} further improve the approach while maintaining the sparse heterogeneity assumption. In contrast, our framework removes the sparsity requirement and instead assumes that arm parameters are drawn from a shared prior distribution, enabling more flexible and adaptive learning of common structures across bandits.

Bayesian methods are particularly well-suited for multi-task learning, as they naturally incorporate prior knowledge and provide principled uncertainty quantification. \citet{hong2022hierarchical} propose a hierarchical Thompson Sampling algorithm (HierTS) that achieves substantially lower regret than methods that ignore shared structure. However, their approach does not address the estimation of the prior covariance matrix --- a key component that remains an open challenge.
\citet{wan2021metadata} introduce a multi-task Thompson Sampling (MTTS) framework based on metadata, along with several computational techniques and efficient variants to improve scalability. Unlike our method, MTTS leverages metadata to capture shared information rather than imposing structural assumptions on the model parameters.
Both HierTS and MTTS fall under the paradigm of task-relation learning, where inter-task relationships are explicitly modeled. In contrast, our approach captures shared structure through a common prior, without explicitly modeling task dependencies.

Multi-task learning has been explored in various bandit-related domains. For example, \citet{gabillon2011multi} and \citet{scarlett2019overlapping} study multi-task approaches for best-arm identification, while \citet{swersky2013multi} extend Bayesian optimization --- a framework for automatic hyperparameter tuning --- by incorporating multi-task Gaussian processes, substantially accelerating the optimization process compared to single-task methods. However, these studies address problems different from ours, as we focus on cumulative regret minimization.

\section{Problem Formulation}
\label{sec:problem_formulation}
We consider a contextual bandit problem with $N$ bandit instances, each associated with $K$ arms. The decision-making process unfolds over $n$ time steps. At each time step $t$, a new individual with a $d$-dimensional context vector $\mathbf{x}_t$ arrives at one of the $N$ bandit instances, determined by a random variable $Z_t$ that follows a categorical distribution with probabilities $p_j$ for $j \in [N]$. The context is drawn from an unknown distribution $\mathbb{P}(\mathbf{x})$. Based on this context vector and historical information, the agent chooses an arm from the $K$ available arms of the selected bandit. The reward $y_{k, j, t}$ corresponding to the chosen arm is then obtained. We assume that the reward is linearly structured, that is, the reward for pulling arm $k$ at time $t$ in instance $j$ is given by
\begin{equation}\label{linear_reward}
	y_{k, j, t} = \mathbf{x}_t^\top \boldsymbol{\beta}_{k,j} + \epsilon_{k,j,t},
\end{equation}
where $\boldsymbol{\beta}_{k,j} \in \mathbb{R}^d$ is the unknown parameter vector for arm $k$ in instance $j$, and $\epsilon_{k,j,t}$ is the noise term, which follows a normal distribution with zero mean and standard deviation $\sigma_k$, i.e., $\epsilon_{k,j,t} \sim \mathcal{N}(0, \sigma_k^2)$. Herein, we account for the heteroscedasticity of arm-specific noise. In the experiments, the noise variance of each individual arm is estimated exclusively using the data collected by that arm itself, thereby ensuring the independence across different arms.

Following the setting in \citet{xu2025multitask}, we consider that individuals arrive at different bandit instances with varying probabilities. 
These arrival probabilities determine the number of samples that each bandit instance receives. We consider the following two different settings: 
\begin{itemize}
	\item Data-balanced setting: The arrival probabilities $p_j$ are approximately equal for all bandit instances $j$. Therefore, each bandit receives a similar number of arrivals over time.
	\item Data-poor setting: One or more bandit instances have significantly lower arrival probabilities compared to others. This setting is particularly useful for evaluating the algorithm's ability to handle cold start problems and leverage multi-task learning to improve efficiency.
\end{itemize}

To enable joint learning across the $N$ bandit instances, we impose a structured prior on the parameters $\boldsymbol{\beta}_{k,j}$ associated with each arm $k$ in bandit instance $j$. Specifically, we assume that the parameters $\boldsymbol{\beta}_{k,j}$ follow a normal distribution centered around a shared parameter vector $\boldsymbol{\beta}_{k0}$: 
\begin{equation*}
	\boldsymbol{\beta}_{k,j} \sim \mathcal{N}(\boldsymbol{\beta}_{k0}, \boldsymbol{\Sigma}_k),
\end{equation*}
where $\boldsymbol{\beta}_{k0}$ represents the shared parameter vector across all $N$ bandit instances and $\boldsymbol{\Sigma}_k$ is the covariance matrix capturing the variability of the parameters for arm $k$. The deviation $\boldsymbol{\beta}_{k,j} - \boldsymbol{\beta}_{k0}$ quantifies the heterogeneity among bandit instances, allowing each instance to have its own unique characteristics while still benefiting from the shared structure.

Our modeling approach differs from \citet{xu2025multitask} in two key aspects.
First, our setting ensures that the shared parameter $\boldsymbol{\beta}_{k0}$ is identifiable, in contrast to \citet{xu2025multitask}, where the global parameter is not identifiable. This identifiability enhances the robustness of our estimation procedure, as non-identifiable models can be highly sensitive to the choice of hyperparameters — making them difficult to tune in practice.
Second, unlike the sparse heterogeneity assumption adopted in \citet{xu2025multitask} --- where only a few components of $\boldsymbol{\beta}_{k,j} - \boldsymbol{\beta}_{k0}$ are assumed to be nonzero --- we assume that $\boldsymbol{\beta}_{k,j} - \boldsymbol{\beta}_{k0}$ follows a zero-mean normal distribution. 
In our simulation studies, our method consistently outperformed the \textsc{RMBandit} algorithm from \citet{xu2025multitask}, even in settings that satisfy the sparse heterogeneity assumption. 
This observation is interesting. One possible explanation is that \textsc{RMBandit} may not fully exploit the underlying sparsity structure, whereas a better-calibrated prior may allow our method to more effectively borrow information from similar tasks, even when the sparse heterogeneity assumption holds.

\begin{figure}[t]
	\centering
	\includegraphics[width=0.8\linewidth]{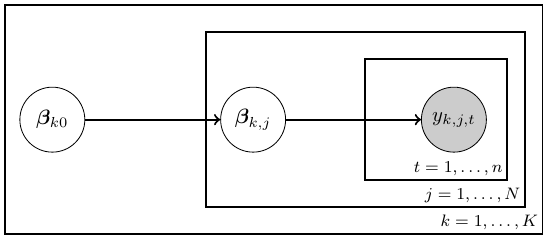}
	\caption{Graphical model of our hierarchical Bayesian framework. The shared $\boldsymbol\beta_{k0}$ serves as the prior mean for each instance-specific $\boldsymbol\beta_{k,j}$, while the deviation $\boldsymbol\beta_{k,j}-\boldsymbol\beta_{k0}$ characterizes how each bandit deviates from the shared structure.}
	\label{fig:graphical_model}
\end{figure}

To evaluate the performance of our sequential decision-making policy, we use cumulative expected regret, a standard metric in the analysis of contextual bandit problems. This metric quantifies the performance gap between our policy and an optimal policy with perfect knowledge of the underlying arm parameters.
Specifically, at each time step $t$, upon observing the bandit $Z_t$, we define an optimal policy $\pi^*_{t}$ that knows the true arm parameters $\{\boldsymbol{\beta}_{k,Z_t}\}_{k\in[K]}$ in advance and always selects the arm with the highest expected reward: 
\begin{equation*}
	\pi^*_{t} = \arg\max_{k\in[K]} \mathbf{x}_t^\top \boldsymbol{\beta}_{k,Z_t},
\end{equation*}
where $\mathbf{x}_t$ is the context vector observed at time $t$. Let $\pi_t$ denote the arm selected by the algorithm at time step $t$ given the bandit $Z_t$. The expected regret at time $t$ for bandit instance $Z_t$ is then
\begin{equation*}
	r_t = \mathbb{E}\left[ \mathbf{x}_t^\top \boldsymbol{\beta}_{\pi^*_{t}, Z_t} - \mathbf{x}_t^\top \boldsymbol{\beta}_{\pi_{t}, Z_t} \right].
\end{equation*} 
The cumulative regret over $n$ time steps is the sum of the per-step regrets: 
\begin{equation*}
	R_n = \sum_{t=1}^n\mathbb{E}\left[ \mathbf{x}_t^\top \boldsymbol{\beta}_{\pi^*_{t}, Z_t} - \mathbf{x}_t^\top \boldsymbol{\beta}_{\pi_{t}, Z_t} \right]. 
\end{equation*}
We also consider the instance-specific cumulative regret for each bandit instance $j$, defined as
\begin{equation*} 
	R_{j,n} =  \sum_{t=1}^{n}\mathbb{E}\left[ \mathbb{I}(Z_t = j)\left( \mathbf{x}_t^\top \boldsymbol{\beta}_{\pi^*_{t}, Z_t} - \mathbf{x}_t^\top \boldsymbol{\beta}_{\pi_{t}, Z_t} \right) \right], 
\end{equation*}
where $\mathbb{I}(\cdot)$ is the indicator function and $n_j= \sum_{t=1}^{n}\mathbb{I}(Z_t=j)$ is the number of arrivals at bandit $j$.
The instance-specific cumulative regret provides insight into our policy's performance for each bandit, especially in data-poor settings, and reflects how effectively the algorithm leverages shared information from other instances.
It is worth noting that \citet{xu2025multitask} define regret based on the oracle policy. In contrast, we define regret with respect to the actual policy $\pi^*_{t}$. A comparison of the two approaches is provided in Appendix \ref{sec:oracle}, where empirical results show very similar performance.

\section{Estimation}
\label{sec:estimation}

We propose a hierarchical Bayesian approach for our model. Specifically, we introduce a prior distribution over the shared parameter $\boldsymbol{\beta}_{k0}$ and the individual deviations $\boldsymbol{\beta}_{k,j} - \boldsymbol{\beta}_{k0}$. This hierarchical structure facilitates information sharing across bandit instances while accommodating instance-specific variations. The hierarchical model is formally defined as follows:
\begin{equation*}
	\begin{aligned}
		y_{k, j, t} \mid \boldsymbol{\beta}_{k,j} & \sim \mathcal{N}\left(\mathbf{x}_{ t}^{\top} \boldsymbol{\beta}_{k,j}, \sigma_k^{2}\right); \\
		\boldsymbol{\beta}_{k,j} \mid \boldsymbol{\beta}_{k 0} & \sim \mathcal{N}\left(\boldsymbol{\beta}_{k0}, \boldsymbol{\Sigma}_{k}\right); \\
		\boldsymbol{\beta}_{k0} & \sim \mathcal{N}\left(\boldsymbol{0}, \lambda^{-1} \mathbf{I}\right).
	\end{aligned}
\end{equation*}	
In Section \ref{subsec:beta}, we present the estimators $\hat{\boldsymbol{\beta}}_{k,j}$ for $\boldsymbol{\beta}_{k,j}$ along with their corresponding covariance matrices. In Section \ref{subsec:mu}, we focus on predicting the expected payoff of arm $k$ in bandit $j$ for a new context $\mathbf{x}_t$, denoted by $\mu_{k,j,t} = \mathbf{x}_t^\top \boldsymbol{\beta}_{k,j}$, and quantifying its associated uncertainty.
In Section \ref{subsec:Sigma}, we describe the estimation of the variance parameters $\boldsymbol{\Sigma}_k$ and $\sigma_k^2$.

\subsection{\texorpdfstring{Estimation of $\boldsymbol{\beta}_{k,j}$ and its Uncertainty}{Estimation of beta_{k,j} and its Uncertainty}}
\label{subsec:beta}

We introduce the estimator $\hat{\boldsymbol{\beta}}_{k,j,t}$ for $\boldsymbol{\beta}_{k,j}$ and quantify its uncertainty under the assumption that $\boldsymbol{\Sigma}_k$ and $\sigma_k^2$ are known. Let $\mathbb{T}_{k,j,t}$ denote the set of time steps when arm $k$ of bandit $j$ is pulled before time step $t$, and define $T_{k,j,t} = |\mathbb{T}_{k,j,t}|$ as its cardinality. Let $\mathbf{y}_{k,j,t} = (y_{k,j,s})_{s\in\mathbb{T}_{k,j,t}}^\top$ be the column vector of observed rewards, $\boldsymbol{\epsilon}_{k,j,t} = (\epsilon_{k,j,s})_{s\in\mathbb{T}_{k,j,t}}^\top$ the corresponding reward noise, and $\mathbf{X}_{k,j,t} = (\mathbf{x}_{s})_{s\in\mathbb{T}_{k,j,t}}^\top$ a $T_{k,j,t} \times d$ matrix of the associated context vectors. 
The model \eqref{linear_reward} can be rewritten in matrix form as
\begin{align*}
	\mathbf{y}_{k,j,t} = \mathbf{X}_{k,j,t}\boldsymbol{\beta}_{k,j} + \boldsymbol{\epsilon}_{k,j,t}.
\end{align*}
Noting that both $\boldsymbol{\beta}_{k,j}$ and the noise are random, the covariance matrix $\mathbf{V}_{k,j,t}$ of the reward vector $\mathbf{y}_{k,j,t}$ is given by
\begin{align*}
	\mathbf{V}_{k,j,t} = \mathbf{X}_{k,j,t}\boldsymbol{\Sigma}_k \mathbf{X}_{k,j,t}^\top + \sigma_k^2 \mathbf{I}_{T_{k,j,t}},
\end{align*}
where $ \mathbf{I}_{T_{k,j,t}}$ is the identity matrix of size $T_{k,j,t}$. In this expression, the first term $\mathbf{X}_{k,j,t}\boldsymbol{\Sigma}_k \mathbf{X}_{k,j,t}^\top$ captures the variability due to $\boldsymbol{\beta}_{k,j} $, while the second term $\sigma_k^2 \mathbf{I}_{T_{k,j,t}}$  accounts for the reward noise.

With these notations in place, we now describe the estimation process in detail. Our procedure unfolds in three steps to effectively utilize the hierarchical structure of the model and ensure accurate arm parameter estimation. First, we use the instance-specific data to estimate the instance-specific $\boldsymbol{\beta}_{k,j}$, conditioned on the shared parameter $\boldsymbol{\beta}_{k0}$. This conditional expectation, denoted as $\widetilde{\boldsymbol{\beta}}_{k,j,t}$, is expressed as a linear function of $\boldsymbol{\beta}_{k0}$. In the second step, we estimate the shared parameter $\boldsymbol{\beta}_{k0}$ using the aggregated data from all instances, denoted as $\hat{\boldsymbol{\beta}}_{k0,t}$. Finally, we substitute the estimated $\hat{\boldsymbol{\beta}}_{k0,t}$ into $\widetilde{\boldsymbol{\beta}}_{k,j,t}$ to obtain the final arm parameter estimate $\hat{\boldsymbol{\beta}}_{k,j,t}$. We now proceed to derive the estimation results.

\paragraph{Step 1: Conditional Expectation of $\boldsymbol{\beta}_{k,j}$}
Given the rewards $\mathbf{y}_{k,j,t}$ that have been obtained and $\boldsymbol{\beta}_{k0}$, the posterior distribution of $\boldsymbol{\beta}_{k,j}$ can be expressed as
\begin{align*}
	\mathbb{P}(\boldsymbol{\beta}_{k,j} \mid \mathbf{y}_{k,j,t}, \boldsymbol{\beta}_{k0}) 
	\propto & \exp\left\{-\frac{1}{2\sigma_k^2}\left\| \mathbf{y}_{k,j,t} -  \mathbf{X}_{k,j,t}\boldsymbol{\beta}_{k,j}\right\|_2^2 -\frac{1}{2} \left( \boldsymbol{\beta}_{k,j} -\boldsymbol{\beta}_{k0}\right)^\top \boldsymbol{\Sigma}_k	\left(\boldsymbol{\beta}_{k,j} -\boldsymbol{\beta}_{k0}\right)\right\}.
\end{align*}
Direct calculation shows that
\begin{equation*}
	\boldsymbol{\beta}_{k,j} \mid \mathbf{y}_{k,j,t}, \boldsymbol{\beta}_{k0} \sim \mathcal{N}(\widetilde{\boldsymbol{\beta}}_{k,j,t}, \sigma_k^{2}\widetilde{\mathbf{C}}_{k,j,t}  ),
\end{equation*}
where
\begin{equation}\label{conditional_estimation_beta_{k,j}}
	\begin{aligned}		&\widetilde{\boldsymbol{\beta}}_{k,j,t} = \sigma_k^{2}\widetilde{\mathbf{C}}_{k,j,t} \boldsymbol{\Sigma}_k^{-1} \boldsymbol{\beta}_{k0} + \widetilde{\mathbf{C}}_{k,j,t} \mathbf{X}_{k,j,t}^\top\mathbf{y}_{k,j,t}; \\
		&\widetilde{\mathbf{C}}_{k,j,t} = (\mathbf{X}_{k,j,t}^\top\mathbf{X}_{k,j,t} + \sigma_k^{2}\boldsymbol{\Sigma}_k^{-1})^{-1}.
	\end{aligned}	
\end{equation}

\paragraph{Step 2: Estimation of $\boldsymbol{\beta}_{k0}$}
Now we derive the posterior of $\boldsymbol{\beta}_{k0}$ given $\{\mathbf{y}_{k,j,t}\}_{j\in[N]}$. We can rewrite the hierarchical Bayesian model as follows: 
\begin{equation*}
	\begin{aligned}
		\mathbf{y}_{k,j,t} \mid \boldsymbol{\beta}_{k0} &\sim \mathcal{N}(\mathbf{X}_{k,j,t}\boldsymbol{\beta}_{k0}, \mathbf{V}_{k,j,t});\\
		\boldsymbol{\beta}_{k0} & \sim \mathcal{N}\left(\boldsymbol{0}, \lambda^{-1} \mathbf{I}\right). 
	\end{aligned}
\end{equation*}
Given the rewards $\{\mathbf{y}_{k,j,t}\}_{j\in[N]}$, the posterior distribution of $\boldsymbol{\beta}_{k0}$ can be expressed as 
\begin{equation*}
	\mathbb{P}(\boldsymbol{\beta}_{k0} \mid \{\mathbf{y}_{k,j,t}\}_{j\in[N]})
	\propto \exp\left\lbrace 
	-\frac{1}{2}\sum_{j=1}^N \left( \mathbf{y}_{k,j,t} -  \mathbf{X}_{k,j,t}\boldsymbol{\beta}_{k0}\right)^\top  \mathbf{V}_{k,j,t}^{-1} \left( \mathbf{y}_{k,j,t} -  \mathbf{X}_{k,j,t}\boldsymbol{\beta}_{k0}\right)
	-\frac{\lambda}{2}\boldsymbol{\beta}_{k0}^\top\boldsymbol{\beta}_{k0}
	\right\rbrace.
\end{equation*}
Then, we have
\begin{equation*}
	\boldsymbol{\beta}_{k0} \mid \{\mathbf{y}_{k,j,t}\}_{j\in[N]} \sim \mathcal{N}(\hat{\boldsymbol{\beta}}_{k0,t}, \boldsymbol{\Phi}_{k0,t}  ),
\end{equation*}
where
\begin{equation*}
	\begin{aligned}
		&\hat{\boldsymbol{\beta}}_{k0,t} = \boldsymbol{\Phi}_{k0,t} \sum_{j=1}^N  \mathbf{X}_{k,j,t}^\top  \mathbf{V}_{k,j,t}^{-1} \mathbf{y}_{k,j,t}; \\
		&\boldsymbol{\Phi}_{k0,t} = \left( \sum_{j=1}^N  \mathbf{X}_{k,j,t}^\top \mathbf{V}_{k,j,t}^{-1} \mathbf{X}_{k,j,t}
		+ \lambda\mathbf{I}\right) ^{-1}.
	\end{aligned}	
\end{equation*}

\paragraph{Step 3: Estimation of $\boldsymbol{\beta}_{k, j}$} 
To derive the posterior expectation and variance of the arm parameter vector $\boldsymbol{\beta}_{k,j}$ given the observed rewards $\{\mathbf{y}_{k,\ell,t}\}_{\ell\in[N]} $, we utilize the law of total expectation and the law of total variance \citep{hong2022hierarchical}, and obtain 
\begin{equation*}
	\boldsymbol{\beta}_{k,j} \mid \{\mathbf{y}_{k,\ell,t}\}_{\ell\in[N]} \sim \mathcal{N}(\hat{\boldsymbol{\beta}}_{k,j,t}, \mathbf{C}_{k,j,t} ),
\end{equation*}
where
\begin{equation}\label{estimation_beta_{k,j}}
	\begin{aligned}
		&\hat{\boldsymbol{\beta}}_{k,j,t} = \sigma_k^{2}\widetilde{\mathbf{C}}_{k,j,t} \boldsymbol{\Sigma}_k^{-1} \hat{\boldsymbol{\beta}}_{k0,t} + \widetilde{\mathbf{C}}_{k,j,t} \mathbf{X}_{k,j,t}^\top\mathbf{y}_{k,j,t};\\
		&\mathbf{C}_{k,j,t} = \sigma_k^{2}\widetilde{\mathbf{C}}_{k,j,t} +\sigma_k^{4}
		\widetilde{\mathbf{C}}_{k,j,t} \boldsymbol{\Sigma}_k^{-1} \boldsymbol{\Phi}_{k0,t}
		\boldsymbol{\Sigma}_k^{-1}\widetilde{\mathbf{C}}_{k,j,t}.
	\end{aligned}	
\end{equation}

Calculating the posterior variance in hierarchical Bayesian models can be computationally intensive, especially when dealing with large datasets. To address this challenge, we employ the Woodbury matrix identity.
In (\ref{estimation_beta_{k,j}}), noting that both $\widetilde{\mathbf{C}}_{k,j,t}$ and $\boldsymbol{\Sigma}_k$ are $d\times d$ matrices, it is easy to compute when $d$ is not too large. The difficulty in computing $\mathbf{C}_{k,j,t}$ mainly lies in $\boldsymbol{\Phi}_{k0,t}$, where $\mathbf{V}_{k,j,t}$ is an $T_{k,j,t} \times T_{k,j,t}$ matrix. As the amount of data increases gradually, the computational difficulty also increases. Recall that $\mathbf{V}_{k,j,t} = \mathbf{X}_{k,j,t}\boldsymbol{\Sigma}_k \mathbf{X}_{k,j,t}^\top + \sigma_k^2 \mathbf{I}_{T_{k,j,t}}$. Applying the Woodbury matrix identity, we have
\begin{align}\label{efficiently compute V}
	\mathbf{V}_{k,j,t}^{-1} 
	&= \sigma_k^{-2}\mathbf{I}_{T_{k,j,t}} - \sigma_k^{-2}\mathbf{X}_{k,j,t}\widetilde{\mathbf{C}}_{k,j,t}\mathbf{X}_{k,j,t}^\top. 	
\end{align}
In (\ref{efficiently compute V}), we transform the problem of inverting a $T_{k,j,t} \times T_{k,j,t}$ matrix into the problem of inverting a $d \times d$ matrix, which significantly accelerates the calculation of the posterior variance, making our estimation process more computationally tractable.

\subsubsection{Estimation error bound analysis}
Now we present the upper bound of our estimation error. Prior to this, it is necessary to propose the following assumption, which is not exceptional and has been widely adopted in existing literature \citep{xu2025multitask, hong2022hierarchical, agrawal2013thompson}. 

\begin{assumption}\label{assumption}
	The ground truth  $\boldsymbol{\beta}_{k,j}$ is bounded in Euclidean norm, i.e., $\|\boldsymbol{\beta}_{k,j}\|_2 \leq c_{\beta}$, for all $k\in[K], j\in[N]$. The context $\mathbf{x}$ is bounded in Euclidean norm, i.e.,  $d^{-1}\|\mathbf{x}\|_{2}^2 \leq c_x$,  for all  $\mathbf{x}$. The eigenvalues of covariance matrix $\boldsymbol{\Sigma}_{k}^{-1}$ are bounded, i.e., $0 < \lambda_d \leq \lambda(\boldsymbol{\Sigma}_{k}^{-1}) \leq \lambda_1$, for all $k\in[K]$. The noise variances of all arms are identical, i.e., $\sigma_k^2 = \sigma^2$, for all $k\in[K]$ (This assumption is made solely for the simplicity of the proof and does not need to be satisfied in subsequent experiments.).
\end{assumption}

It is worth noting that our assumption regarding the boundedness of $\mathbf{x}$ requires that its infinity norm be bounded (i.e., $\|\mathbf{x}\|_\infty <\infty$). Since the $\ell_2$ and $\ell_{\infty}$ norms are related by the inequality $\|\mathbf{x}\|_\infty \leq \|\mathbf{x}\|_2 \leq \sqrt{d} \cdot \|\mathbf{x}\|_\infty$, this assumption directly yields the requisite bounds utilized throughout the subsequent technical analysis.

\begin{theorem}\label{inequality_beta_hat}
	Under Assumption~\ref{assumption}, the estimator $\hat{\boldsymbol{\beta}}_{k,j,t}$ that incorporates prior information satisfies the following inequality with probability at least $1-\delta$, for any fixed $t\geq 1$ and $\mathbf{x} \in \mathbb{R}^d$: 
	\begin{equation*}
		\begin{aligned}
			&\| \widetilde{\mathbf{C}}_{k,j,t}^{-\frac{1}{2}}
			\left(\hat{\boldsymbol{\beta}}_{k,j,t} -\boldsymbol{\beta}_{k,j} \right) \|_2 \leq \sigma\alpha_{t}(\delta), \\
			&\left|  \mathbf{x}^\top
			\left(\hat{\boldsymbol{\beta}}_{k,j,t} -\boldsymbol{\beta}_{k,j} \right) \right|  \leq  \alpha_{t}(\delta)\|\mathbf{x}\|_{\mathbf{C}_{k,j,t}},
		\end{aligned}
	\end{equation*}
	where
	\begin{align*}
		\alpha_{t}(\delta)= 2\sqrt{d m_1 \log \left(\frac{m_2 + t c_x}{m_3\delta} \right) } + Nm_4
		=  O(\sqrt{d\log (t/\delta)}),
	\end{align*}
	$m_1=\max\{\lambda_1/\lambda, 1\}$, $m_2=\max\{\lambda, \sigma^2\lambda_1\}$, 
	$m_3=\sqrt{\lambda\lambda_d \sigma^2}/2$ and $m_4=c_{\beta}\sqrt{\lambda_1}$.
\end{theorem}

Theorem~\ref{inequality_beta_hat} provides an explicit characterization of the upper bound for the estimation error under the incorporation of prior information. Most existing studies on multi-bandit problems \citep{hong2022hierarchical, aouali2023mixed} focus on analyzing Bayesian regret as the primary performance metric. Building on Theorem~\ref{inequality_beta_hat}, we derive the frequentist regret in Section~\ref{sec:regret_analysis} and fill this research gap by establishing rigorous theoretical results for the frequentist upper bound.

We briefly outline the proof for Theorem~\ref{inequality_beta_hat} here; full details are provided in Appendix~\ref{appendix_thm1}. The estimation error is decomposed into two terms:
\begin{align*}
	\widetilde{\mathbf{C}}_{k,j,t}^{-\frac{1}{2}}
	\left(\hat{\boldsymbol{\beta}}_{k,j,t} - \boldsymbol{\beta}_{k,j} \right)  
	= \sigma^2 \widetilde{\mathbf{C}}^{\frac{1}{2}}_{k,j,t} \boldsymbol{\Sigma}_{k}^{-1}\left(\hat{\boldsymbol{\beta}}_{k0,t} - \boldsymbol{\beta}_{k,j} \right)
	+ \widetilde{\mathbf{C}}^{\frac{1}{2}}_{k,j,t}\mathbf{X}_{k,j,t}\boldsymbol{\epsilon}_{k,j,t},
\end{align*}
where the first term represents the deviation due to the prior-informed estimator, and the second term captures the randomness introduced by observational noise. 
The noise term is bounded using the theoretical result established in \citet{abbasi2011improved}, stated in Lemma \ref{abbasi_inequality}. 
The primary challenge in establishing Theorem~\ref{inequality_beta_hat} lies in handling the first term, which involves the prior-informed estimation error $\hat{\boldsymbol{\beta}}_{k0,t}$. Its $L_2$-norm can be bounded by that of the prior-induced error, $\hat{\boldsymbol{\beta}}_{k0,t} - \boldsymbol{\beta}_{k,j}$, which consists of a term involving the matrix norms of the true parameters across all bandits and a noise term. We carefully examine the boundedness of these norms and apply the theoretical results from \citet{abbasi2011improved} to establish the bound.

\subsection{\texorpdfstring{Prediction of $\mu_{k,j,t}$ under Context $\mathbf{x}_t$ and its Uncertainty}{Prediction of mu_{k,j,t} under Context x_t and its Uncertainty}}
\label{subsec:mu}
From (\ref{estimation_beta_{k,j}}), the predicted reward for a new context $\mathbf{x}_t$ is given by
\begin{align}\label{mu}
	\hat{\mu}_{k,j,t} = \mathbf{x}_t^\top \hat{\boldsymbol{\beta}}_{k,j,t}
	= \sigma_k^{2}\mathbf{x}_t^\top\widetilde{\mathbf{C}}_{k,j,t} \boldsymbol{\Sigma}_k^{-1} \hat{\boldsymbol{\beta}}_{k0,t} + \mathbf{x}_t^\top\widetilde{\mathbf{C}}_{k,j,t} \mathbf{X}_{k,j,t}^\top\mathbf{y}_{k,j,t}.
\end{align}
Now we measure the uncertainty of $\hat{\mu}_{k,j,t}$ in (\ref{mu}) for efficient exploration. We measure the uncertainty by the mean squared error $\mathbb{E}[(\hat{\mu}_{k,j,t} - \mu_{k,j,t})^2 \mid \{\mathbf{y}_{k,j,t}\}_{j\in[N]} ]$, where $\mu_{k,j,t}=\mathbf{x}_t^\top \boldsymbol{\beta}_{k,j}$. Then we have
\begin{equation}\label{tau}
	\begin{aligned}
		\mathbb{E}\left[(\hat{\mu}_{k,j,t} - \mu_{k,j,t})^2 \mid \{\mathbf{y}_{k,j,t}\}_{j\in[N]}\right]
		= \mathbf{x}_t^\top \mathbf{C}_{k,j,t} \mathbf{x}_t =:\tau_{k,j,t}^2. 
	\end{aligned}	
\end{equation}

\subsection{\texorpdfstring{Estimation of $\boldsymbol{\Sigma}_k$ and $\sigma_k^2$}{Estimation of Sigma_k and sigma_k^2}}
\label{subsec:Sigma}
Both the estimator $\hat{\mu}_{k,j,t}$ and its uncertainty $\tau_{k,j,t}^2$ depend on $\boldsymbol{\Sigma}_k$ and $\sigma_k^2$. When these parameters are unknown, they can be estimated from the data. Substituting the estimates of $\boldsymbol{\Sigma}_k$ and $\sigma_k^2$ into the expressions for $\hat{\mu}_{k,j,t}$ and $\tau_{k,j,t}^2$ yields the empirical Bayesian estimators. This approach is referred to as empirical Bayesian multi-bandit learning.

We estimate $\sigma_k^2$  by
\begin{equation*}
	\hat{\sigma}_{k,t}^2 = \frac{\sum\nolimits_{j=1}^N \|\mathbf{y}_{k,j,t} - \mathbf{X}_{k,j,t}\hat{\boldsymbol{\beta}}_{k,j,t}\|_2^2}{\max\{\sum_{j=1}^N T_{k,j,t}-d-1,1\}}.
\end{equation*}

Estimating $\boldsymbol{\Sigma}_k$ involves the problem of covariance matrix estimation. A common approach is to use the sample covariance matrix. However, when the dimensionality of the variables increases with the number of bandit instances $N$, the estimation becomes a high-dimensional problem. 
To address this issue, we adopt the covariance matrix estimation method proposed by \citet{bickel2008covariance}, which is based on the following sparsity assumption: $\boldsymbol{\Sigma}_k$ belongs to a class of covariance matrices defined by 
$$
\begin{aligned}
	\mathcal{C}_{q}\left(c_{0}(d), M, M_{0}\right) = &\left\{\boldsymbol{\Sigma}: \sigma_{i i} \leq M, \sum_{j=1}^{d}\left|\sigma_{i j}\right|^{q} \leq c_{0}(d),  \forall i;\ \ \lambda_{\min }(\boldsymbol{\Sigma}) \geq M_{0}>0\right\},
\end{aligned}	
$$
where $0 \leq q < 1$. Denote $\hat{\boldsymbol{\beta}}_{k,j,t}^{\text{ols}} =(\mathbf{X}_{k,j,t}^\top\mathbf{X}_{k,j,t})^{-1}\mathbf{X}_{k,j,t}^\top\mathbf{y}_{k,j,t}$. We obtain the sample covariance matrix
\begin{equation*}
	\mathbf{S}_{k,t} = \frac{1}{N}\sum_{j=1}^N \hat{\boldsymbol{\beta}}_{k,j,t}^{\text{ols}}\hat{\boldsymbol{\beta}}_{k,j,t}^{\text{ols}\top} - \frac{1}{N^2}\sum_{j=1}^N \hat{\boldsymbol{\beta}}_{k,j,t}^{\text{ols}}\sum_{j=1}^N \hat{\boldsymbol{\beta}}_{k,j,t}^{\text{ols}\top}. 
\end{equation*}
Therefore, the threshold matrix estimator is defined as
\begin{equation*}
	\hat{\boldsymbol{\Sigma}}_{k,t} = (s_{ij} \mathbb{I}(|s_{ij}| \geq \gamma)), 
\end{equation*}
where $s_{ij}$ denotes the element in the $i$-th row and $j$-th column of the matrix $\mathbf{S}_{k,t}$, and $\mathbb{I}(\cdot)$ is the indicator function, which takes the value $1$ if the condition is satisfied and $0$ otherwise. 
We follow the approach of \citet{bickel2008covariance} to select the thresholding parameter $\gamma$. In their work, \citet{bickel2008covariance} establish the following consistency result. 
\begin{lemma}
	Assume that  $\boldsymbol{\Sigma}_k \in \mathcal{C}_{q}\left(c_{0}(d), M, M_{0}\right) $ and $\frac{\log d}{N}=o(1) $, then we have that 
	\begin{equation*}
		\left\|\hat{\boldsymbol{\Sigma}}_{k,t}-\boldsymbol{\Sigma}_{k}\right\|=O_{p}\left(c_{0}(d)\left(N^{-1} \log d\right)^{(1-q) / 2}\right). 
	\end{equation*}	
\end{lemma}
The use of the thresholded covariance matrix estimator is a key component of our approach. It provides an automatic, data-driven approach for identifying meaningful correlations across bandit instances. Broadly speaking, the method retains strong correlations between instances, while setting weak correlations to zero --- effectively removing those that are insufficient to support across-instance learning.

\section{Algorithms}
	\label{sec:algorithm}
	We propose two algorithms for Empirical Bayesian Multi-Bandit (ebm). Based on the posterior distribution of the arm parameters derived in Section~\ref{subsec:beta}, we first introduce a sampling-based algorithm that leverages the posterior of $\boldsymbol{\beta}_{k,j}$.
	We refer to this algorithm as \textsc{ebmTS}, which stands for the posterior sampling strategy. It is worth noting, however, that \textsc{ebmTS} may not represent true posterior sampling, as the prior parameters $\boldsymbol{\Sigma}_k$ and $\sigma_k^2$ are estimated using a frequentist approach. 
	The second algorithm follows the UCB framework and builds on the \textsc{ReUCB} algorithm introduced in \citet{zhu2022random,zhu2022robustcontextuallinearbandits}. We refer to this variant as \textsc{ebmUCB}, which incorporates both random effects and contextual information in a unified framework.

	\begin{algorithm}[!h]
		\caption{\textsc{ebmTS} and \textsc{ebmUCB} in Empirical Bayesian Multi-Bandit.}
		\label{ebmTS_ebmUCB}
		\begin{algorithmic}[1]
			\STATE{\textbf{Input: } hyperparameters $\lambda, a$.}
			\FOR{$t = 1,2,\ldots,n$}
			\STATE{Observe an arrival at instance $Z_t = j$.}
			\STATE{Observe context $\mathbf{x}_t$.}
			\FOR{$k = 1,2,\ldots, K$}
			\STATE{Obtain $\hat{\mu}_{k,j,t}$ from (\ref{mu}) and $\hat{\tau}_{k,j,t}^2$ from (\ref{tau}).}
			\STATE{Define
				\begin{equation*}
					\begin{aligned}
						\textsc{ebmTS: } & U_{k,j,t} = \mathbf{x}_t^\top \breve{\boldsymbol{\beta}}_{k,j,t} \text{ with } \breve{\boldsymbol{\beta}}_{k,j,t} \sim \mathcal{N}(\hat{\boldsymbol{\beta}}_{k,j,t}, \alpha^2_t(\delta) \mathbf{C}_{k,j,t} ),\, \text{OR}\\
						\textsc{ebmUCB: } & U_{k,j,t} = \hat{\mu}_{k,j,t} + \alpha_t(\delta) \tau_{k,j,t}.
					\end{aligned}
				\end{equation*}
			}
			\ENDFOR
			\STATE{\textbf{if } $t \leq K$ \textbf{ then } $\pi_{t} \gets t$ \textbf{ else } $\pi_{t} \gets \arg\max_{k\in[K]} U_{k,j,t}$.}
			\STATE{Pull action $\pi_{t}$ of bandit $j$ and observe reward $y_{\pi_{t},j,t}$.}
			\STATE{Update $\hat{\boldsymbol{\beta}}_{\pi_{t},\ell,t}$ and $\mathbf{C}_{\pi_{t},\ell,t}$ for $\ell \in [N]$, and update $\hat{\boldsymbol{\Sigma}}_{\pi_{t}}$, $\hat{\sigma}_{\pi_{t}}^2$.}
			\ENDFOR
		\end{algorithmic}
	\end{algorithm}

	We summarize both \textsc{ebmTS} and \textsc{ebmUCB} in Algorithm~\ref{ebmTS_ebmUCB}. 
	Given a new context $\mathbf{x}_t$ at time step $t$, we apply (\ref{mu}) and (\ref{tau}) to compute the predicted reward: 
	\begin{equation*}
		\hat{\mu}_{k,j,t} = \mathbf{x}_t^\top \hat{\boldsymbol{\beta}}_{k,j,t}, 
	\end{equation*}
	and its corresponding uncertainty	
	\begin{equation*}
		\tau_{k,j,t}^2 = \mathbf{x}_t^\top \mathbf{C}_{k,j,t}\mathbf{x}_t	. 
	\end{equation*}	
	
	We now describe the two algorithms in detail.
	\paragraph{\textsc{ebmTS}}
	We presented the posterior distribution in (\ref{estimation_beta_{k,j}}), which is $\boldsymbol{\beta}_{k,j} \mid \{\mathbf{y}_{k,\ell,t}\}_{\ell\in[N]} \sim \mathcal{N}(\hat{\boldsymbol{\beta}}_{k,j,t}, \mathbf{C}_{k,j,t} )$.  
	This allows us to sample directly from the posterior distribution. At time step $t$, the sampling-based selection rule is given by 
	\begin{equation*}
		\pi_{j,t} = \arg\max_{k\in[K]} \mathbf{x}_t^\top \breve{\boldsymbol{\beta}}_{k,j,t}
		\text{ with }
		\breve{\boldsymbol{\beta}}_{k,j,t} \sim \mathcal{N}\left(\hat{\boldsymbol{\beta}}_{k,j,t}, \alpha^2_t(\delta) \mathbf{C}_{k,j,t} \right),    
	\end{equation*}
	It should be noted that $\alpha_t(\delta)$ is of the order of $\sqrt{\log t}$. Therefore, in practical experiments, we set $\alpha_t(\delta) = a\sqrt{\log t}$, where $a>0$ is a tunable hyperparameter to control the degree of exploration.
	
	\paragraph{\textsc{ebmUCB}}
	We adopt a UCB-based exploration strategy to address the uncertainty in estimating $\hat{\mu}_{k,j,t}$ for sequential decision-making. The exploration bonus for arm $k$ of bandit $j$ at time step $t$ is given by $\alpha_t(\delta)\tau_{k,j,t}$. 
	Accordingly, the selection rule of \textsc{ebmUCB} at time step $t$ is 
	\begin{equation*}
		\pi_{j,t} = \arg\max_{k\in[K]} U_{k,j,t} \text{ with }    U_{k,j,t} = \hat{\mu}_{k,j,t} + \alpha_t(\delta)\tau_{k,j,t},
	\end{equation*}
	where, similar to \textsc{ebmTS}, in practical experiments, we set the same $\alpha_t(\delta)$ as above.

	\section{Regret Analysis}
	\label{sec:regret_analysis}
	
	In this section, we present the frequentist cumulative regret upper bounds for the algorithms \textsc{ebmUCB} and \textsc{ebmTS} , which differ from the Bayesian regret bound reported in previous literature \citep{hong2022hierarchical}. 
	\begin{theorem}\label{regret_bound_ebmUCB}
		For any $\delta > 0$, the overall cumulative regret of \textsc{ebmUCB} is at most
		$$
		\begin{aligned}
			R_n 
			\leq &\ 2\alpha_{n}(\delta) \sqrt{nd K \left[ c_1 \sum_{j=1}^N\log (1 + c_2 n_j)+c_3 \log (1+ c_4 N) \right]} 
			+ 2 \sqrt{d c_x}c_{\beta}KNn \delta \\
			= &\ O\left( d^{3/2}\sqrt{nK\log\frac{n}{\delta} \left(\sum_{j=1}^N \log n_j + d\log N \right) }\right) ,
		\end{aligned}
		$$
		where  $c_1 = \frac{ \lambda_d^{-1}  dc_x}{ \log (1+ \sigma^{-2} \lambda_d^{-1}  dc_x)} = O(d)$, $c_2 = \sigma^{-2}\lambda_d^{-1}c_x$, $c_3 = \frac{\lambda_d^{-2} \lambda_1^{2} \lambda^{-1} dc_x (1 + \sigma^{-2} \lambda_d^{-1}  dc_x)}{ \log (1+ \sigma^{-2} \lambda_d^{-2} \lambda_1^{2} \lambda^{-1} dc_x)} = O(d^2)$, and $c_4 = \lambda_1 \lambda^{-1}$.
	\end{theorem}
	
	\begin{theorem}\label{regret_bound_ebmTS}
		For any $\delta > 0$, the overall cumulative regret of \textsc{ebmTS} is at most
		$$
		\begin{aligned}
			R_n &\leq 2C \alpha^{\prime}_n(\delta) \sqrt{nd K \left[ c_1 \sum_{j=1}^N\log (1 + c_2 n_j)+c_3 \log (1+ c_4 N)\right]} 
			+ 2 \sqrt{d c_x}c_{\beta}KN (n\delta + \pi^2/6) \\
			&= O\left( d^{2}\sqrt{nK\log\frac{n}{\delta} \log n\left(\sum_{j=1}^N \log n_j + d\log N \right) }\right) ,
		\end{aligned}
		$$
		where $C=\max _{t \geq 1} \frac{2}{\left|\frac{1}{2  \sqrt{2\pi e}}-\frac{1}{t^{2}}\right|}+1$, $\alpha^{\prime}_n(\delta) = ( \gamma_n+1)\alpha_{n}(\delta)$, and $\gamma_n=\sqrt{2d\log (dn^2)}$. 
	\end{theorem}

	Theorems~\ref{regret_bound_ebmUCB} and \ref{regret_bound_ebmTS} establish the frequentist upper bounds on regret, where each component admits a clear interpretation: $\alpha_n(\delta)$ represents the contribution from the estimation error bound; $\sqrt{c_1 nd K\sum_{j=1}^N\log \left(1 + c_2 n_j\right)}$ reflects the regret associated with learning instance-specific parameters; and $\sqrt{c_3 ndK \log \left(1+ c_4 N\right)}$ captures the regret incurred in learning the arm-level prior parameters. 
	These results show that the cost of learning the shared prior grows only logarithmically with the number of instances, while the instance-specific learning cost depends on the individual sample sizes through $\sum\nolimits_{j=1}^N\log \left(1 + c_2 n_j\right)$. This is consistent with our empirical results, where information shared through the common prior improves performance across instances, particularly when some instances have limited data.

	Compared with the bound in Theorem~\ref{regret_bound_ebmUCB} for \textsc{ebmUCB}, the regret bound in Theorem~\ref{regret_bound_ebmTS} for \textsc{ebmTS} includes an additional factor, $\gamma_n$, arising from sampling error, reflecting the inherent variability introduced by its sampling-based nature.
	
	Compared with the Bayesian regret upper bound for multi-bandit problems \citep{hong2022hierarchical}, our frequentist regret bounds exhibit a similar structural form. A key difference, however, lies in the scaling: the Bayesian bound scales with $d^{\frac{3}{2}}$ because Bayesian regret inherently accounts only for sampling error and disregards estimation error. 
	In contrast, our algorithm \textsc{ebmTS}, when analyzed under the frequentist regret framework, must account for this estimation error. This necessity introduces an additional factor of $\sqrt{d}$, leading to a higher-order dependence on $d$. Despite this, our bounds remain competitive. Crucially, our algorithms exhibit superior empirical performance, which we attribute to their ability to leverage shared structure for more efficient learning.

	The proofs of Theorems~\ref{regret_bound_ebmTS} and \ref{regret_bound_ebmUCB} follow three main steps. First, we decompose the cumulative regret into two components: the estimation error and the variance summation term, $\mathcal{V}_n$, defined as 
	$\mathcal{V}_n =  \sum_{t=1}^n \mathbf{x}_t^\top \mathbf{C}_{\pi_{t}, Z_t, t} \mathbf{x}_t$. 
	Next, we apply Theorem~\ref{inequality_beta_hat} to control the estimation error. Finally, Lemmas~\ref{variance_sum_1} and \ref{variance_sum_2} in Appendix~\ref{appendix_lemma_1} are used to bound the variance term $\mathcal{V}_n$.

	\section{Experiments}
	\label{sec:experiment}
	
	In this section, we present experiments on both synthetic and real-world datasets to evaluate the performance of our proposed algorithms. In all experiments, we set $\lambda = 0.001$ and $a = 0.1$ for both \textsc{ebmTS} and \textsc{ebmUCB}. 
	A sensitivity analysis of these hyperparameters is provided in Appendix~\ref{appendix:hyperparams}. The results show that the performance is robust to their values within reasonable ranges. In practice, these hyperparameters can be tuned if necessary.
	To ensure robustness, each experiment is repeated with 100 different random seeds. 
	All computations are performed on a computer equipped with an Apple M3 processor, 24 GB of memory, and \textsc{OS X} operating system.
	
	We compare our methods against the following baseline algorithms:
	\begin{itemize}
		\item \textsc{RMBandit} \citep{xu2025multitask}: It is designed for multi-task learning under the assumption of sparse heterogeneity across bandit instances. For hyperparameters, we take $\eta_0 = \eta_{1,0} = 0.2$, $h=15$ and $q=50$, the same as \citep{xu2025multitask}.
		\item \textsc{OLSBandit} \citep{goldenshluger2013linear}: It uses ordinary least squares (OLS) regression to estimate the parameters for each bandit instance independently, without leveraging any shared structure. For hyperparameters, we take $h=15$ and $q=1$, the same as \citet{xu2025multitask}.
		\item \textsc{LinTS} \citep{agrawal2013thompson}: It is based on a sampling strategy but it does not perform multi-task learning.
		\item \textsc{LinUCB} \citep{li2010contextual}: It uses a UCB-based exploration but does not perform multi-task learning.
	\end{itemize}

	\subsection{Experiments on Synthetic Data}
	We generate synthetic data to simulate a multi-bandit environment using the hierarchical Bayesian model. 
	The shared parameters $\boldsymbol{\beta}_{k0}$ are drawn from a normal distribution $\mathcal{N}(\mathbf{0}, \mathbf{I})$, and the instance-specific parameters $\boldsymbol{\beta}_{k,j}$ are then drawn from $\mathcal{N}(\boldsymbol{\beta}_{k0}, \boldsymbol{\Sigma}_k)$, where the covariance matrices  $\boldsymbol{\Sigma}_k$ are constructed as $\mathbf{b}\mathbf{b}^\top + \mathbf{I}$, with $\mathbf{b}\sim\mathcal{N}(\mathbf{0}, \mathbf{I})$. 
	Context vectors $\mathbf{x}_t$ are drawn from a mixture of Gaussian distributions. Specifically, each element of $\mathbf{x}_t$ is independently sampled from $\mathcal{N}(-1, 1)$ with probability $0.5$, and from $\mathcal{N}(1, 1)$ with probability $0.5$.
	We use Gaussian noise with a standard deviation of $\sigma_k = 1$ for all arms. In the data-balanced setting, we assign equal sampling probabilities with $p_j = 1/N$ for all $j$. In the data-poor setting, we set $p_1 = 0.1p_i$ and $p_i = p_j$ for all $i, j \geq 2$, creating a scenario in which the first bandit receives significantly fewer samples than the others.
	
	\begin{figure}[!t]
		\centering
		\subfigure[Cumulative regret in the data-balanced setting]{
			\includegraphics*[width=0.85\textwidth]{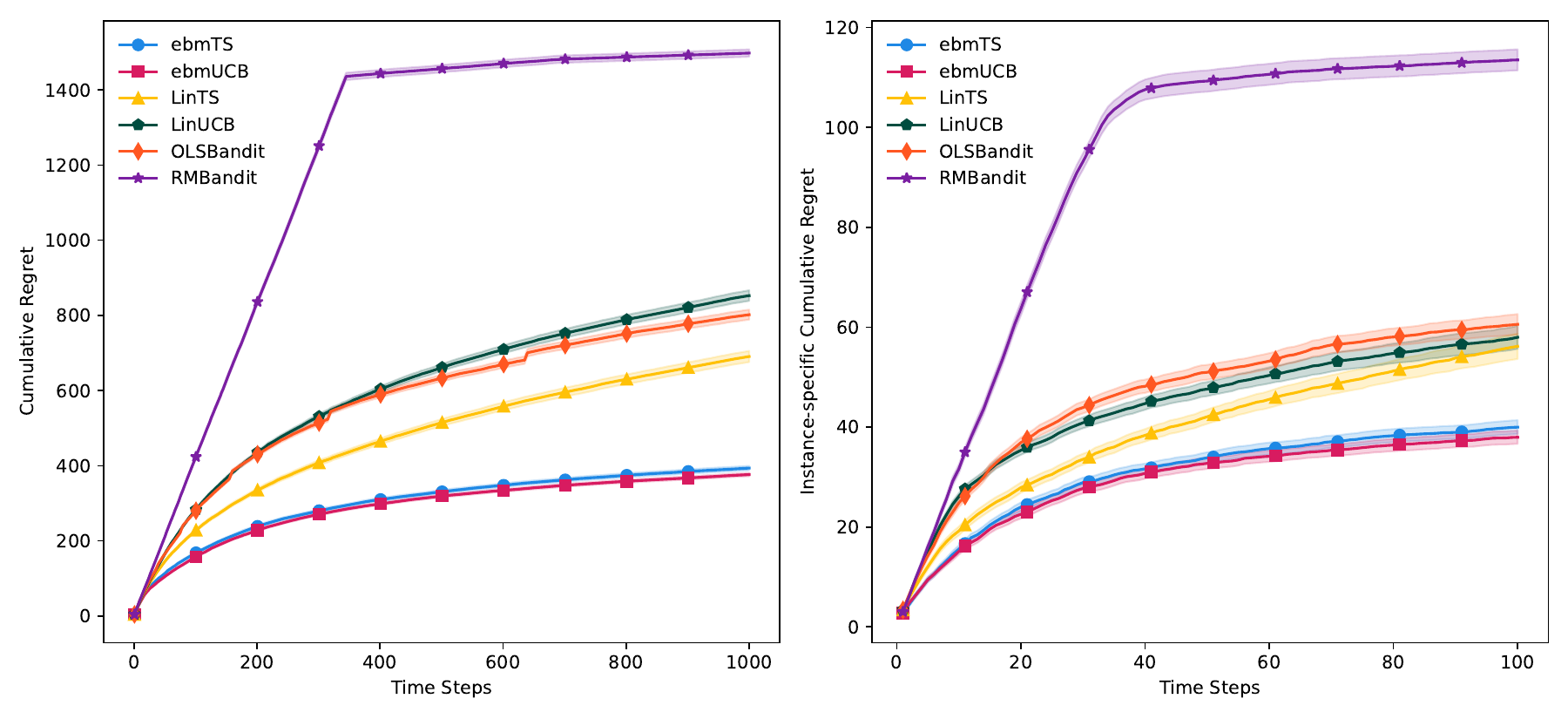}
			\label{fig:data-balanced}
		}\\
		\subfigure[Cumulative regret in the data-poor setting]{
			\includegraphics*[width=0.85\textwidth]{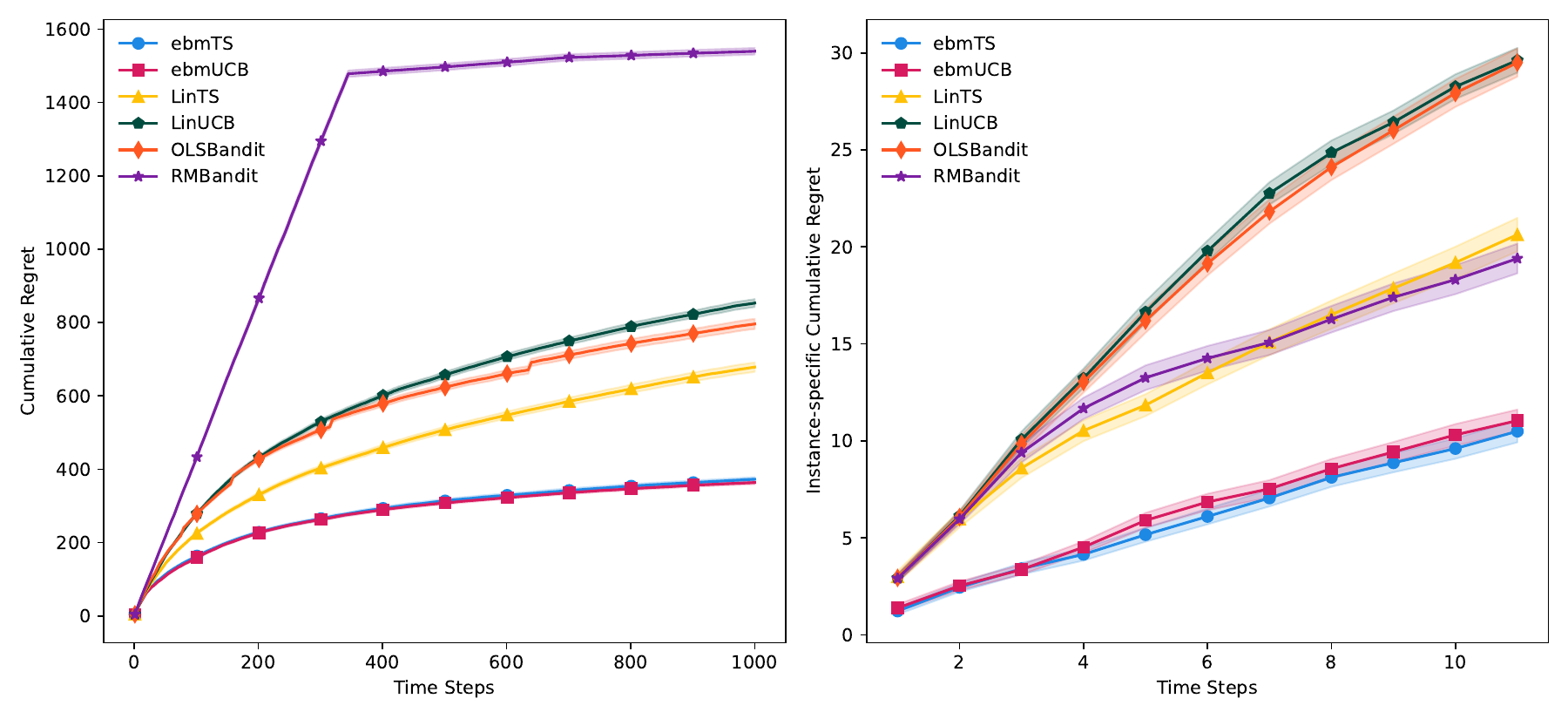}
			\label{fig:data-poor}		
		}
		\caption{Performance under $N=10$, $K=5$, $d=3$.}
		\label{fig:synthetic_regret_curve_plot}
	\end{figure}

	\begin{figure}[ht]
		\centering
		\includegraphics*[width=0.5\textwidth]{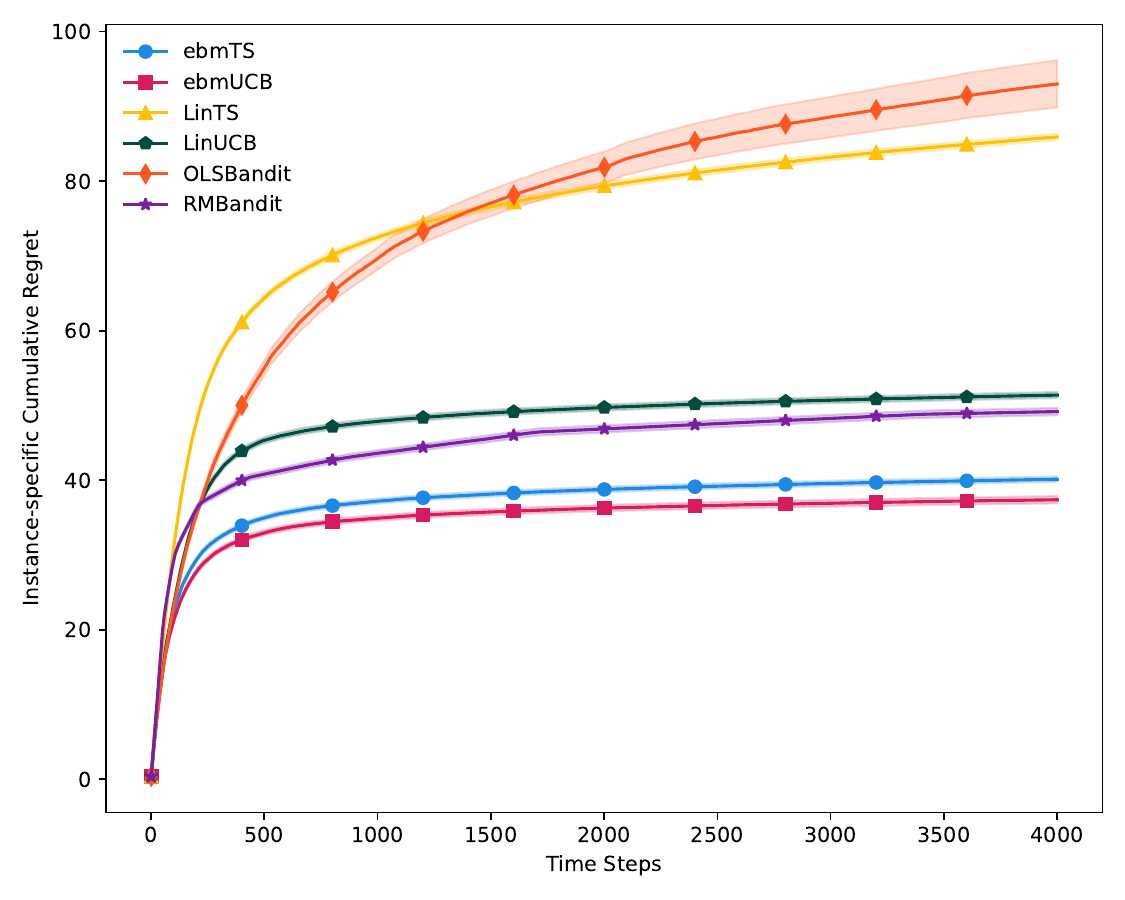}
		\caption{Performance under Sparse heterogeneity.}
		\label{fig:synthetic_regret,sparse_heterogeneity}
	\end{figure}

	\begin{figure}[!t]
		\centering
		\subfigure[Vary context distribution]{
			\includegraphics*[width=0.45\textwidth]{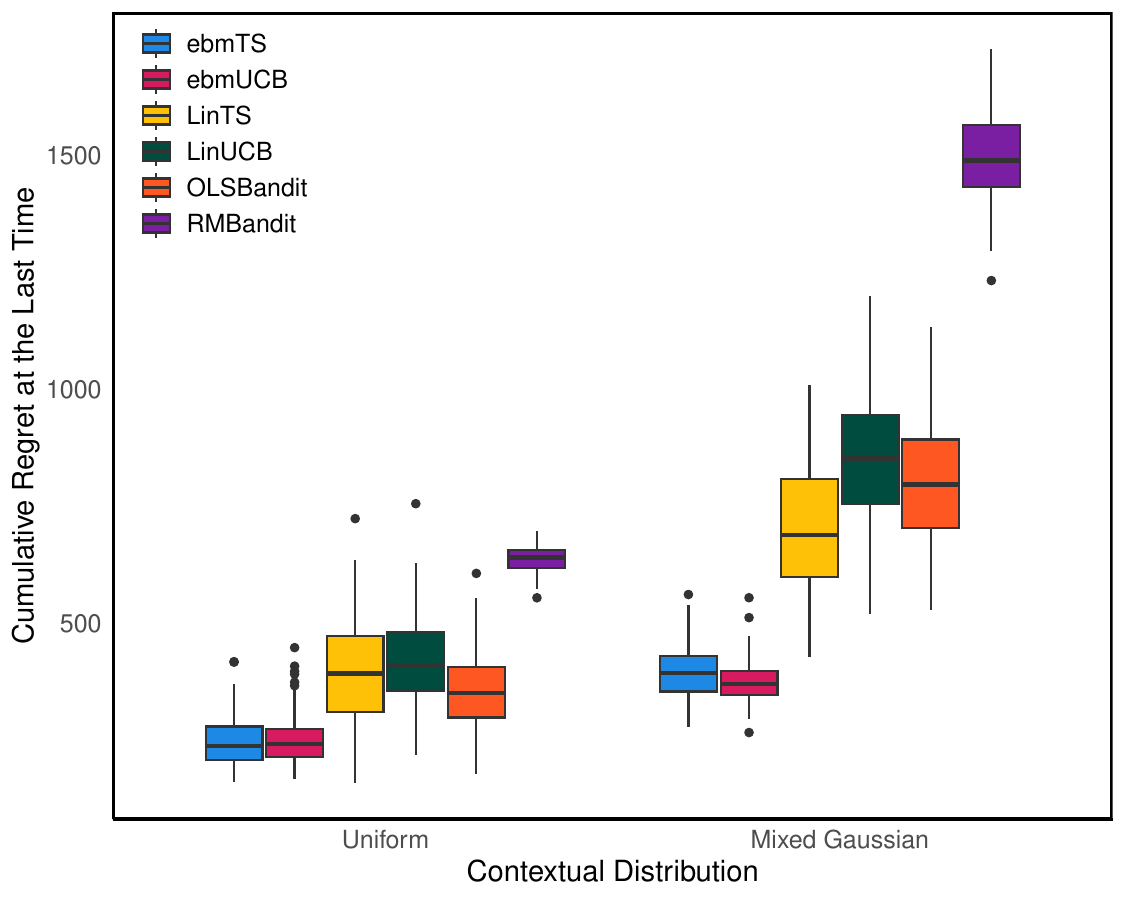}
			\label{fig:vary_context_distribution}
		}
		\subfigure[Vary the number of bandits $N$]{
			\includegraphics*[width=0.45\textwidth]{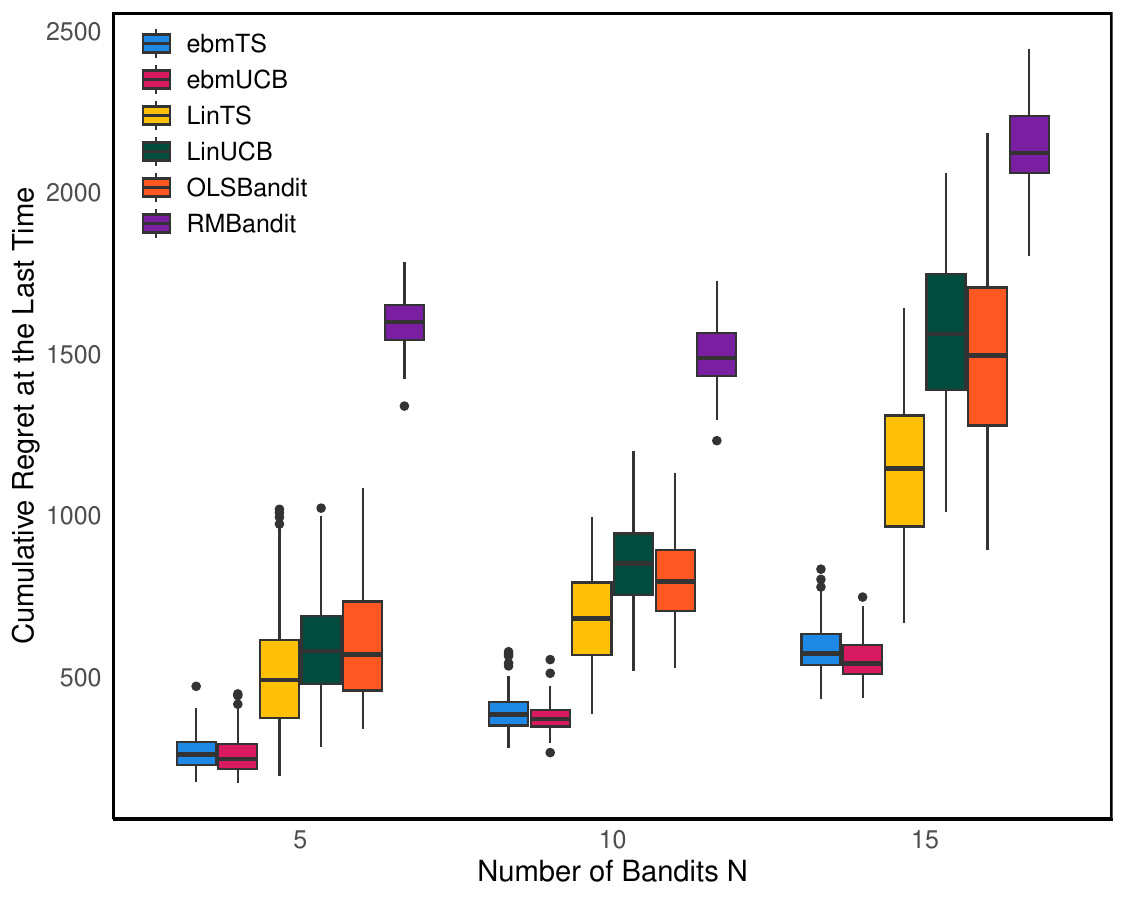}
			\label{fig:vary_bandits_N}		
		}
		\subfigure[Vary the number of arms $K$]{
			\includegraphics*[width=0.45\textwidth]{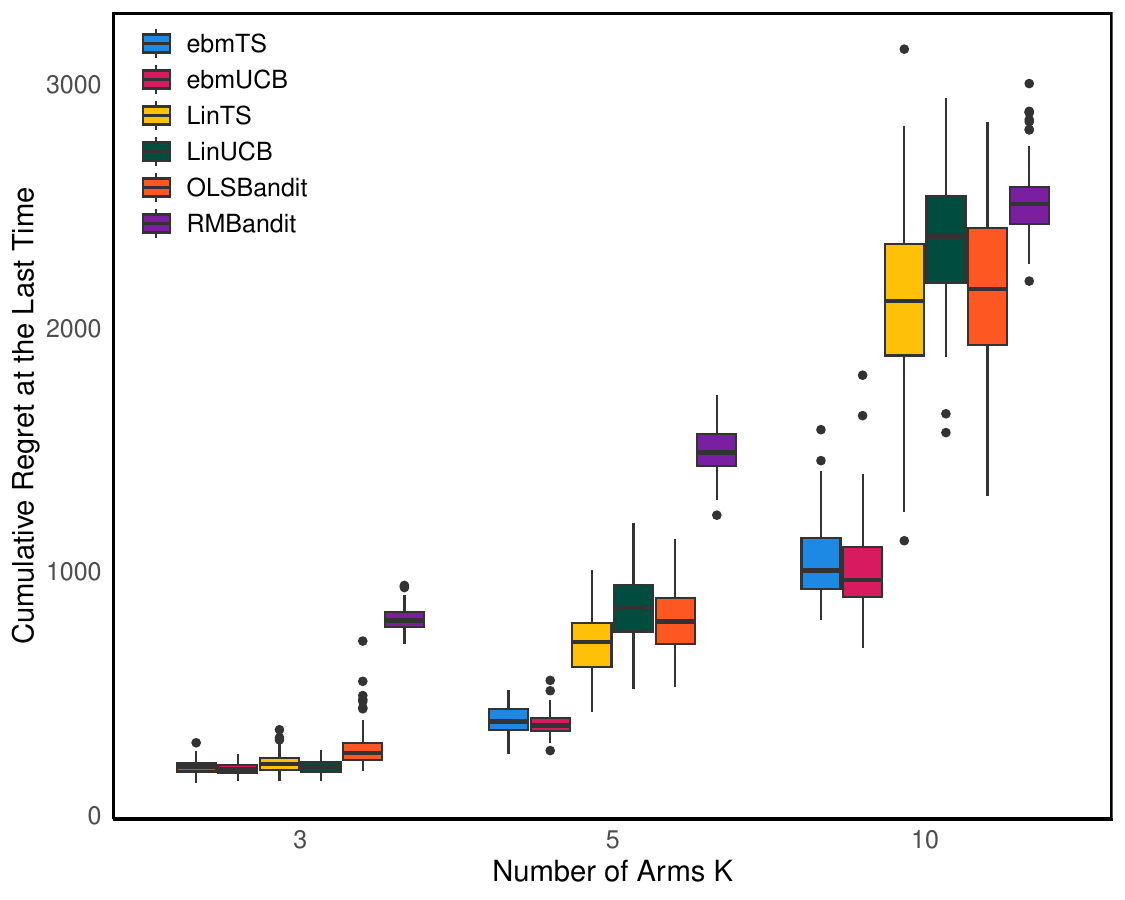}
			\label{fig: vary arms K}		
		}
		\subfigure[Vary dimension $d$]{
			\includegraphics*[width=0.45\textwidth]{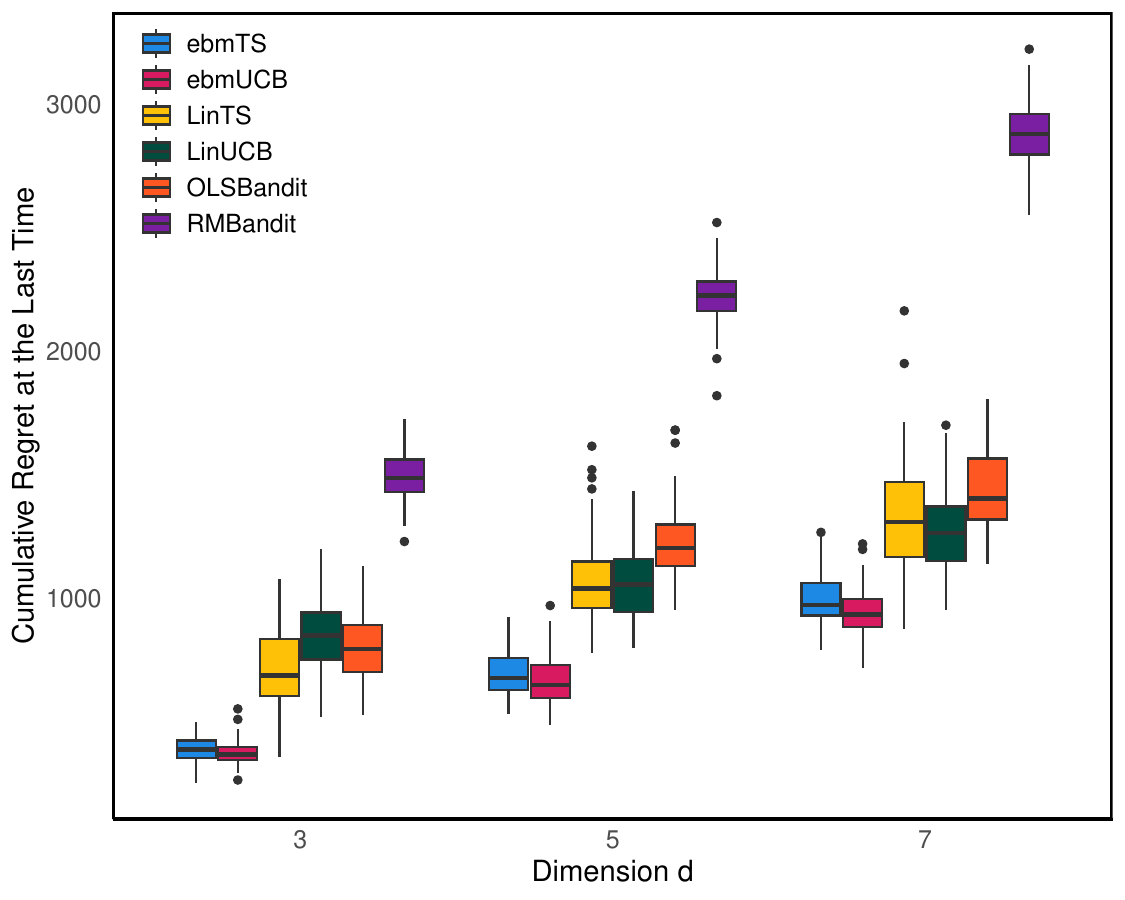}
			\label{fig:vary_dimension_d}		
		}
		\caption{Performance under different context distributions, $N$, $K$, and $d$.}
		\label{fig:synthetic_regret_bar_plot,vary_K_N_d}
	\end{figure}

	In Figure~\ref{fig:data-balanced}, we present the cumulative regret and the instance-specific cumulative regret for the first bandit as an example in the data-balanced setting, respectively. 
	Similarly, in Figure~\ref{fig:data-poor}, we show the corresponding results for the data-poor setting, where the first bandit receives fewer samples. 
	The plots in Figure~\ref{fig:synthetic_regret_curve_plot} demonstrate that our algorithms, \textsc{ebmTS} and \textsc{ebmUCB}, consistently achieve significantly lower cumulative regret compared to the baseline methods. By effectively leveraging shared information across bandit instances, our algorithms exhibit faster convergence and reduced regret over time in both settings. 
	In contrast, \textsc{LinTS}, \textsc{LinUCB}, and \textsc{OLSBandit} do not share information across bandit instances, resulting in higher cumulative regret that continues to grow over time due to less efficient exploration and exploitation. Although \textsc{RMBandit} incorporates information sharing, its performance remains suboptimal. It initially emphasizes exploration, leading to a substantial accumulation of regret in the early stages. 
	It is worth noting that the cumulative regret curves under the data-poor and data-balanced settings are nearly identical in shape. Each arm shares a common prior across instances, while the data-poor instance contributes negligibly to the aggregate regret. Thus, the total regret is dominated by the data-rich instances, leaving its overall shape largely unchanged.
	
	Our model framework is built on Gaussian assumptions regarding the coefficients across arms. We now investigate whether these Gaussian assumptions restrict its applicability in broader settings. Specifically, we consider a sparse heterogeneity assumption, similar to the setup in \citet{xu2025multitask}, under which our hierarchical model is misspecified.
	Figure~\ref{fig:synthetic_regret,sparse_heterogeneity} reports the cumulative regret results. The results show that our algorithms, \textsc{ebmTS} and \textsc{ebmUCB}, continue to outperform \textsc{RMBandit} even under the sparse heterogeneity assumption. This demonstrates that our methods can effectively exploit shared information for joint learning and capture contextual signals, despite model misspecification. Overall, these results highlight the robustness and broad applicability of our approach.

	We also evaluate the performance of the algorithms under various parameter settings, including different context distributions, numbers of bandit instances ($N$), numbers of arms ($K$), and dimensions of the context vectors ($d$). The corresponding results are presented in Figure~\ref{fig:synthetic_regret_bar_plot,vary_K_N_d}.
	
	In Figure~\ref{fig:vary_context_distribution}, we present the performance of the algorithms across different context distributions, focusing on the uniform and mixed Gaussian settings. The uniform distribution, where each context element is sampled uniformly from $[-1, 1]$, represents a simple continuous case that satisfies the covariate diversity condition discussed in \citet{bastani2021mostly}. Under such conditions, exploration requirements can be significantly reduced --- potentially leading to near-exploration-free behavior.
	In contrast, the mixed Gaussian distribution is more complex and violates the covariate diversity condition. In this setting, \textsc{ebmTS} and \textsc{ebmUCB} achieve notably lower cumulative regret, demonstrating their robustness in handling complex and heterogeneous contextual information.
	
	The effects of other key parameters --- including the number of bandit instances ($N$), the number of arms ($K$), and the context dimensionality ($d$) --- are presented in Figures~\ref{fig:vary_bandits_N}-\ref{fig:vary_dimension_d}, respectively. 
	As $N$ increases, the number of tasks grows, but so does the pool of shared information. Our algorithms effectively exploit this shared structure, consistently outperforming the baselines. Similarly, as $K$ and $d$ increase --- making the learning problem more challenging --- \textsc{ebmTS} and \textsc{ebmUCB} maintain strong performance. Between the two, \textsc{ebmTS} and \textsc{ebmUCB} exhibit comparable average cumulative regret, though \textsc{ebmUCB} tends to show lower variability across trials. This stability is evident in the box plots, where \textsc{ebmUCB} displays narrower interquartile ranges, indicating more consistent performance.

	\subsection{Experiments on Real-World Data}
	
	\paragraph{SARCOS Dataset}
	The SARCOS dataset\footnote{\url{https://gaussianprocess.org/gpml/data/}} addresses a multi-output learning problem for modeling the inverse dynamics of a SARCOS anthropomorphic robot with seven degrees of freedom. Each sample includes 21 input features --- comprising seven joint positions, seven joint velocities, and seven joint accelerations. This dataset has been widely used in the literature \citep{balduzzi2015compatible, zhang2021survey}, and contains 44,484 training examples and 4,449 test examples. We treat each output (i.e., degree of freedom) as a separate arm. To simulate a multi-task bandit environment, we first apply linear regression to the test dataset to estimate the model parameters and their variances. These estimates are then used to generate task-specific parameters for the $N=30$ bandit instances we consider. Finally, we evaluate the performance of our algorithms on the training dataset to assess their effectiveness.
	
	\paragraph{Activity Recognition Dataset}
	The UCI dataset titled "Activity recognition with healthy older people using a batteryless wearable sensor" \footnote{\url{https://archive.ics.uci.edu/dataset/427/activity+recognition+with+healthy+older+people+using+a+batteryless+wearable+sensor}} is designed to monitor the activities of healthy elderly individuals with the aim of reducing the occurrence of harmful events, such as falls. The dataset provides $60$ *.csv files for room $1$ and $28$ *.csv files for room $2$. Each *.csv file includes eight features from the W2ISP (Wearable Wireless Identification and Sensing Platform) sensor and the RFID (Radio Frequency Identification) reader, along with the label for each record. 
	The labels indicate activities such as sitting on the bed, sitting on a chair, lying in bed, and walking. We focus on Room 1 due to its larger data volume. 
	Through one-hot encoding of categorical features and principal component analysis (PCA) for dimensionality reduction on the features, an $8$-dimensional feature set is finally obtained.
	After filtering out files with insufficient data, we retain $16$ files, each treated as a separate task. The activity labels are interpreted as arms in the bandit problem framework. For each file, we split the data into training and testing sets using a $30\%$-$70\%$ ratio. A hierarchical Bayesian model is fitted to the training data to estimate the underlying environment parameters, which are then used to evaluate the algorithms on the corresponding test sets.
	
	\paragraph{MovieLens $10$M Dataset}
	The MovieLens dataset, which is widely used in the research of contextual bandits \citep{cella2020meta, christakopoulou2018learning, hong2023multi, wan2021metadata}, comes in various sizes to accommodate different research needs. The MovieLens $10$M dataset \footnote{\url{https://grouplens.org/datasets/movielens/10m/}} \citep{harper2015movielens}, contains $10$ million ratings for $10{,}677$ movies from $69{,}878$ users. As a first step, we complete the sparse rating matrix using singular value decomposition
	(SVD) with rank $d = 10$. 
	More specifically, let $\mathbf{R}$ denote the rating matrix, where the element at the $i$-th row and $j$-th column represents the rating given by user $i$ to movie $j$. Applying SVD to $\mathbf{R}$ yields the approximation $\mathbf{R} \approx \mathbf{U}  \mathbf{V}^\top $. Here, the $i$-th row of $\mathbf{U}$, denoted as $\mathbf{u}_i$, represents the features of user $i$, while the $j$-th row of $\mathbf{V}$, denoted as $\mathbf{v}_j$, represents the features of movie $j$. The rating of movie $j$ by user $i$ is then obtained by the dot product $\mathbf{u}_i^\top\mathbf{v}_j$.
	Then we apply a Gaussian mixture model (GMM) with $K = 10$ clusters to the rows of $\mathbf{V}$. 
	We set the prior parameters to the center and covariance estimated by the Gaussian Mixture Model (GMM). And we generate the parameter vectors for each task to simulate similar tasks. We set the number of tasks, $N$, to be $10$.
	
	\paragraph{Letter Recognition Dataset}
	The letter recognition dataset \footnote{\url{https://archive.ics.uci.edu/dataset/59/letter+recognition}}  \citep{letter_recognition} is a multi-class classification dataset containing handwritten capital letters from various writers. It has been widely used in prior studies \citep{deshmukh2017multi, wang2019batch}. The dataset consists of 20,000 samples, each representing a capital letter from the English alphabet. Each sample is described by 16 numerical features that capture various attributes of the letter's shape and structure. We split it into training and test sets with a ratio of $30\%$ for training and $70\%$ for testing. 
	For each category (arm), we fit a linear regression model on the training data to estimate the corresponding coefficients and variances. These estimates are then used to generate parameter vectors for each task, enabling the simulation of related tasks. In this experiment, we set $N = 30$.

	\begin{figure}[!t]
		\centering
		\subfigure[SARCOS Dataset]{
			\includegraphics*[width=0.45\textwidth]{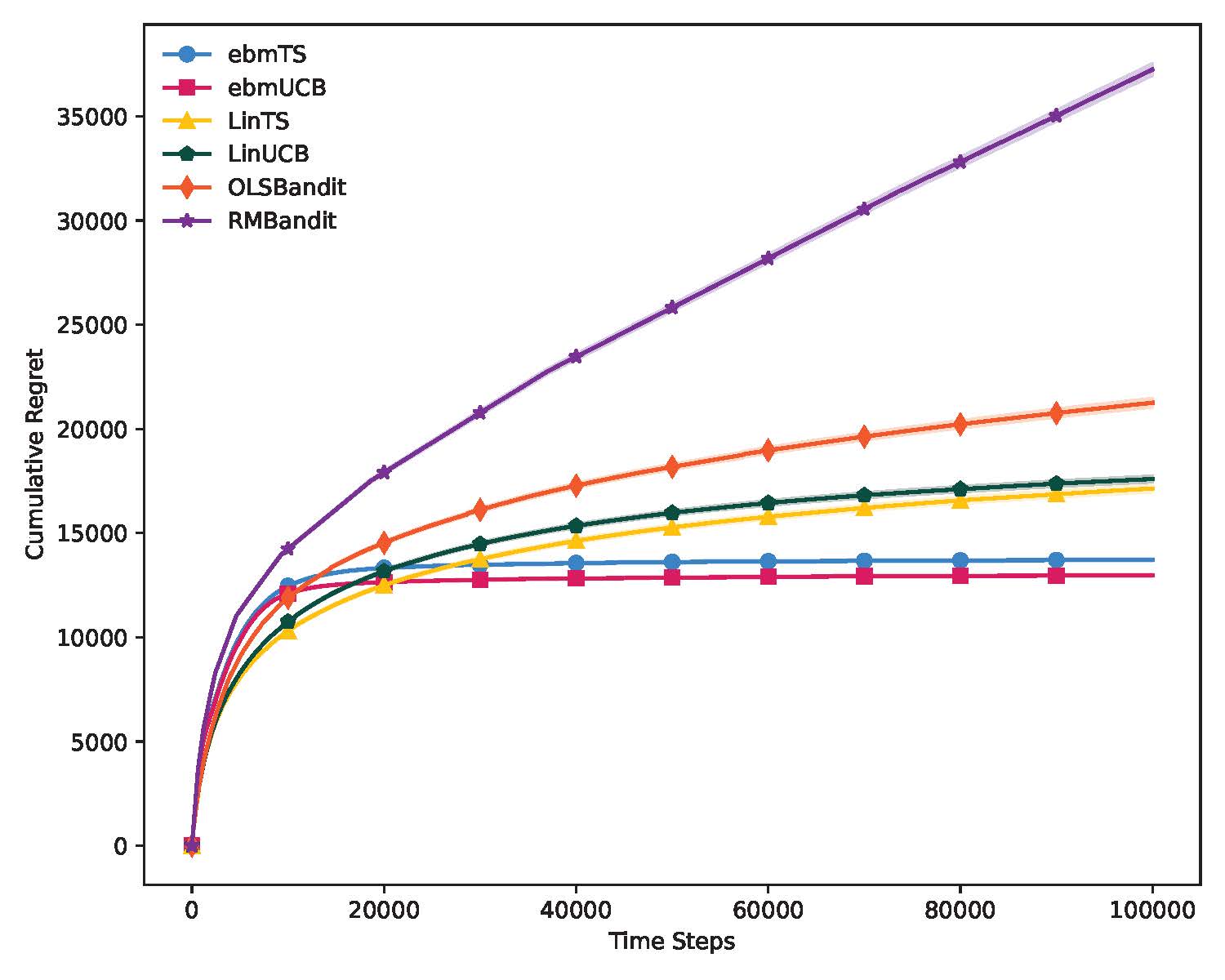}
			\label{fig: sarcos}
		}
		\subfigure[Activity Recognition Dataset]{
			\includegraphics*[width=0.45\textwidth]{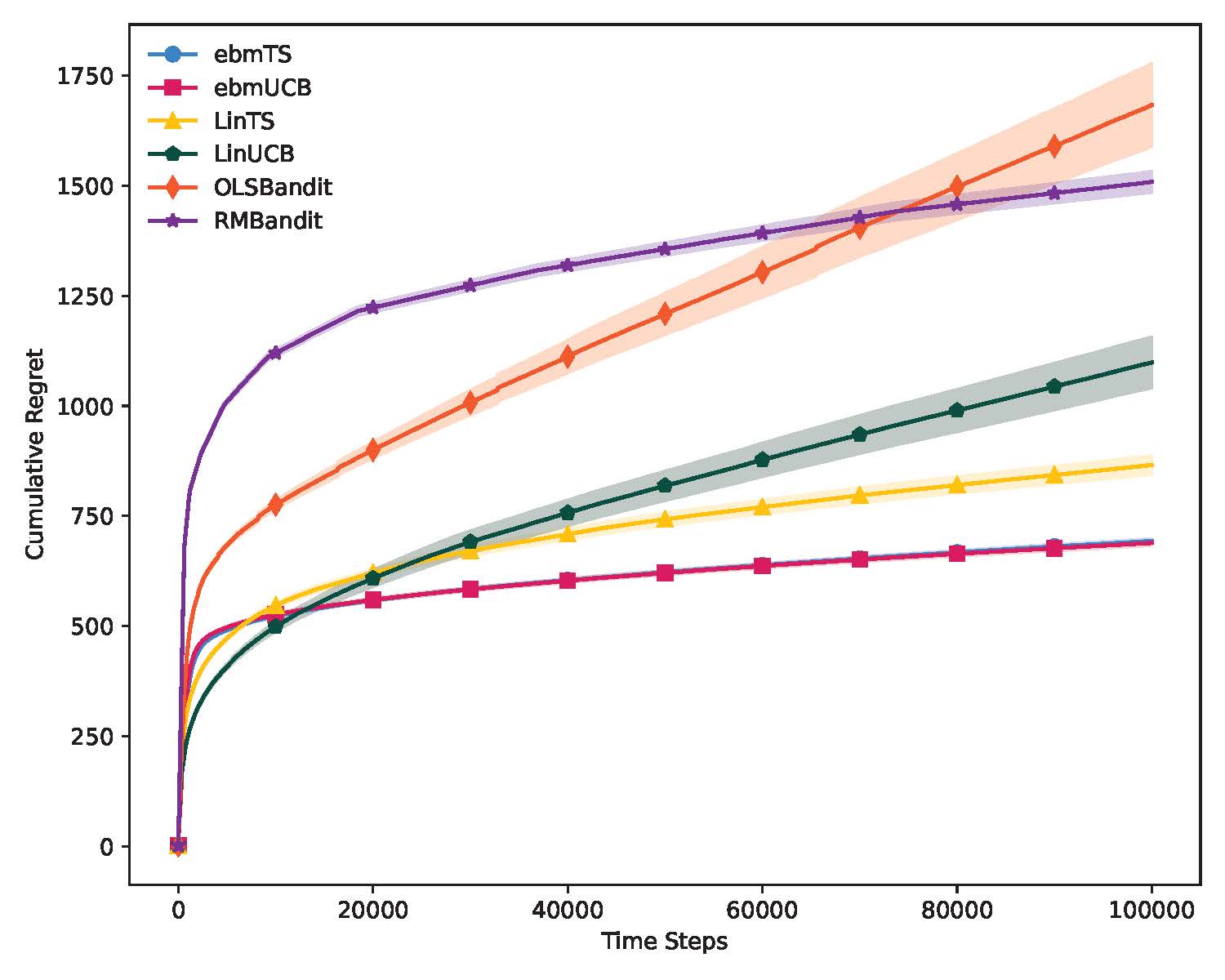}
			\label{fig: activity recognition}		
		}
		\subfigure[MovieLens $10$M Dataset]{
			\includegraphics*[width=0.45\textwidth]{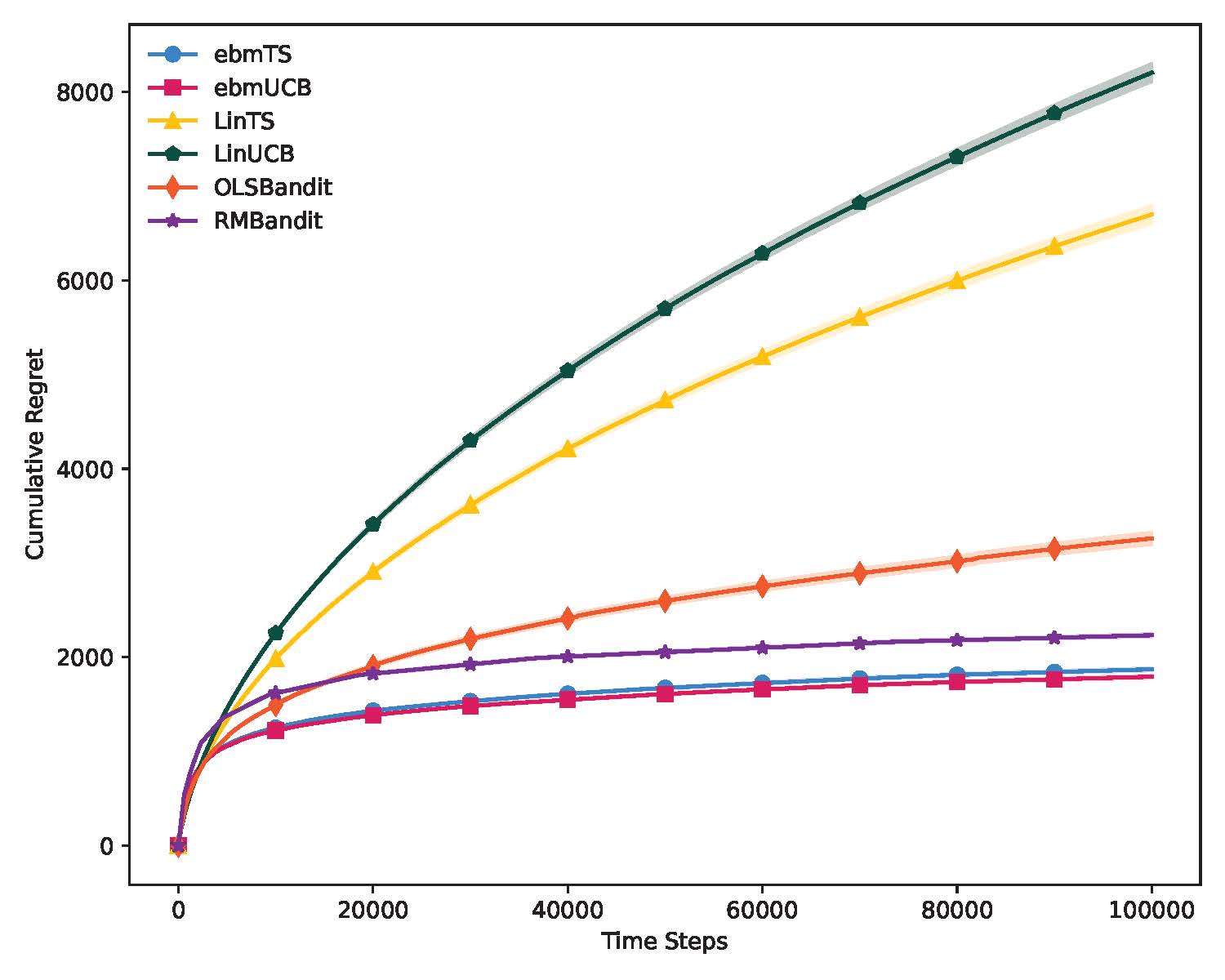}
			\label{fig: movielens}		
		}
		\subfigure[Letter Recognition Dataset]{
			\includegraphics*[width=0.45\textwidth]{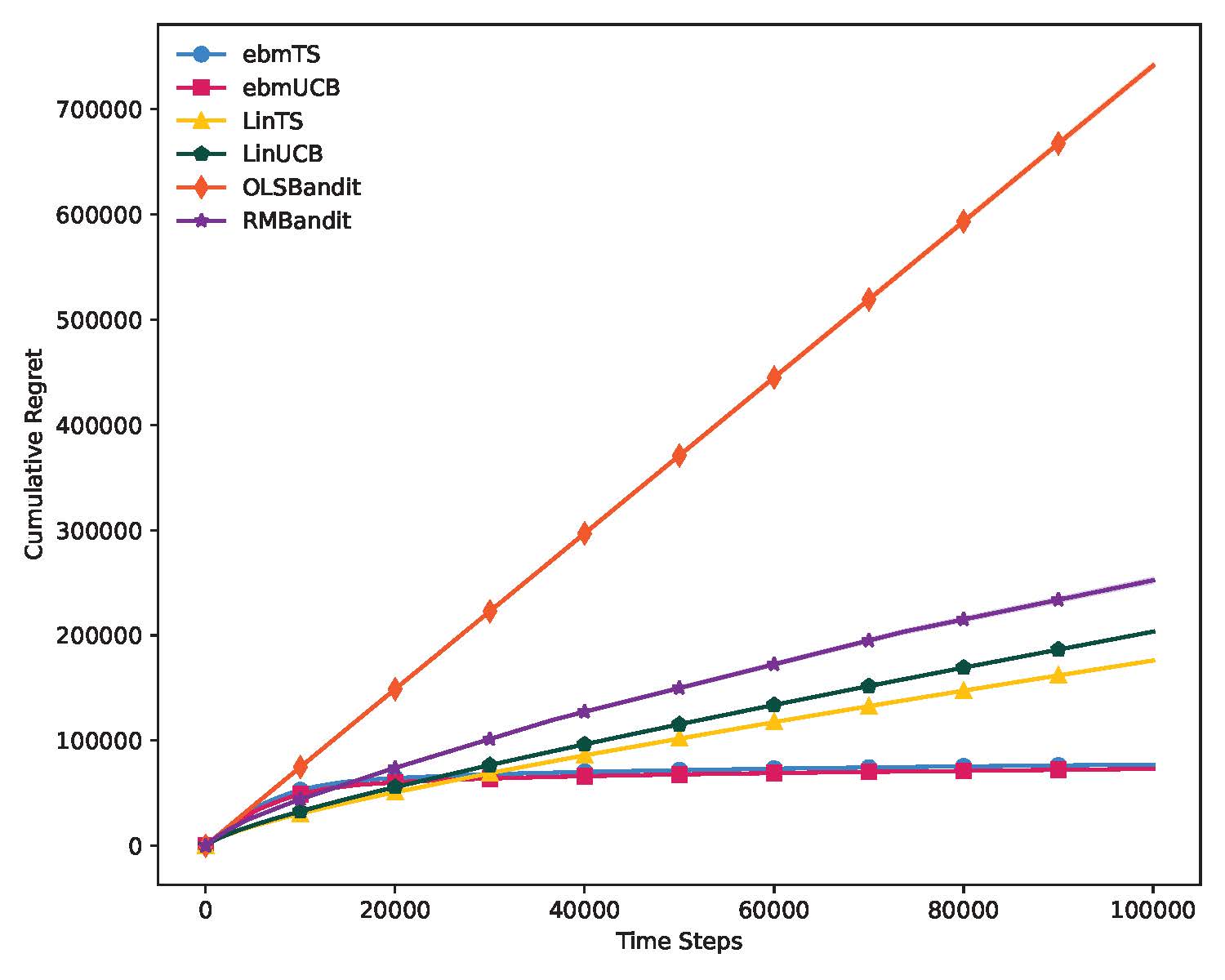}
			\label{fig: letter recognition}		
		}
		\caption{Performance of real-world datasets.}
		\label{fig:real_data}
	\end{figure}

	We evaluate the performance of our algorithms, \textsc{ebmTS} and \textsc{ebmUCB}, on these four real-world datasets to assess their practical effectiveness and generalizability, as shown in Figure~\ref{fig:real_data}. Across diverse application domains, both algorithms consistently outperform the baseline methods. This superior performance highlights their ability to efficiently share information across related tasks, adapt to heterogeneous contextual structures, and maintain robustness under varying data complexities. These results demonstrate the practical value of our framework for real-world decision-making scenarios.

	\section{Conclusion}
	\label{sec:conclusion}
	In this article, we introduced a novel hierarchical Bayesian framework that captures shared structures across multiple bandit instances while accounting for instance-specific variations. Building on this framework, we developed an empirical Bayesian multi-task approach for contextual bandits, addressing the limitations of existing methods and resulting in more efficient algorithms --- \textsc{ebmTS} and \textsc{ebmUCB} --- suitable for real-world applications. 
	
	Unlike prior methods that rely on the sparse heterogeneity assumption \citep{xu2025multitask}, our approach imposes no such restrictions, allowing for more flexible learning across bandit instances. In addition, unlike previous hierarchical models \citep{hong2022hierarchical, wan2021metadata}, we explicitly estimate the prior covariance matrix, providing both practical guidance for implementation and deeper insight into the shared structure among arms. Our theoretical analysis, based on frequentist regret, rigorously characterizes the performance of the proposed algorithms, while extensive experiments on synthetic and real-world datasets demonstrate that \textsc{ebmTS} and \textsc{ebmUCB} consistently outperform existing methods in cumulative and instance-specific regret. These results highlight the ability of our framework to handle complex, heterogeneous environments, bridging the gap between theoretical guarantees and practical applicability.

	However, there are several limitations to our work that merit discussion. First, our paper assumes Gaussian distributions and linear models for analytical convenience. While the framework could be extended to more general distributions and complex models, doing so would require approximate methods for posterior sampling. Second, although our framework effectively leverages shared information across tasks, it assumes that relationships between tasks are relatively consistent. This assumption may not hold in dynamic environments with rapidly changing task-specific dynamics. Extending our approach to such dynamic settings represents an important direction for future research. 
	Third, our theoretical analysis relies on the prior covariance matrix being full-rank. When $\mathrm{rank}(\boldsymbol{\Sigma}_k) = r < d$, the arm parameters lie in an $r$-dimensional subspace, with no variation along the remaining $d-r$ null-space directions. Extending the analysis to this low-rank setting would reveal the potential benefits of exploiting the underlying low-rank structure.
	Finally, the scalability of our methods may be challenging in large-scale environments, where high-dimensional feature spaces or a large number of bandit instances can lead to substantial computational overhead due to the complexity of managing and updating hierarchical structures and prior covariance matrices. More efficient implementations will need to be explored in future work. 
	
		\section*{Acknowledgements}
		The authors would like to thank the Associate Editor, Marek Petrik, and the three anonymous reviewers for their invaluable comments, which greatly improved this manuscript. 
		Zhu's work was partially supported by the National Institutes of Health (NIH) Intramural Research Program under grant number: OT2OD040373.  
		The findings and conclusions presented in this paper are those of the author(s) and do not necessarily reflect the views of the NIH or the U.S. Department of Health and Human Services.
		
	
	\bibliography{Bibliography}
	\bibliographystyle{ims}

	\appendix

\section{Proof of Theorem \ref{inequality_beta_hat}}\label{appendix_thm1}

In this section, we provide the detailed proof of Theorem \ref{inequality_beta_hat}. We begin by introducing two supporting lemmas. Lemma \ref{beta_hat-beta} decomposes the estimation error $\hat{\boldsymbol{\beta}}_{k,j,t} -\boldsymbol{\beta}_{k,j}$  into two components: the first corresponds to the error induced by the prior estimation $\hat{\boldsymbol{\beta}}_{k0,t}$, and the second arises from the observation noise. Lemma \ref{XVX} reformulates the term $\mathbf{X}_{k,j,t}^\top \mathbf{V}_{k,j,t}^{-1} \mathbf{X}_{k,j,t}$ used in the construction of the prior-informed estimator $\hat{\boldsymbol{\beta}}_{k0,t}$.

\begin{lemma}\label{beta_hat-beta}
	The estimation error $\hat{\boldsymbol{\beta}}_{k,j,t} - \boldsymbol{\beta}_{k,j}$ can be explicitly decomposed into two interpretable components: the first captures the discrepancy arising from the prior-informed estimator, while the second reflects the uncertainty due to observational noise, as shown below.
	$$
	\begin{aligned}
		\hat{\boldsymbol{\beta}}_{k,j,t} - \boldsymbol{\beta}_{k,j} 
		=&\ \underbrace{ \sigma^2 \widetilde{\mathbf{C}}_{k,j,t} \boldsymbol{\Sigma}_{k}^{-1}\left(\hat{\boldsymbol{\beta}}_{k0,t} - \boldsymbol{\beta}_{k,j} \right)}_{\text{Error from the Prior-Informed Estimator}} + \underbrace{\widetilde{\mathbf{C}}_{k,j,t}\mathbf{X}_{k,j,t}\boldsymbol{\epsilon}_{k,j,t}}_{\text{Error Induced by Observational Noise}}.
	\end{aligned} 
	$$
\end{lemma}
\begin{proof}
	Recalling that $\mathbf{y}_{k,j,t} = \mathbf{X}_{k,j,t}\boldsymbol{\beta}_{k,j} + \boldsymbol{\epsilon}_{k,j,t}$, and substituting this expression into $\hat{\boldsymbol{\beta}}_{k,j,t}$, we obtain
	\begin{equation*}
		\begin{aligned}
			\hat{\boldsymbol{\beta}}_{k,j,t} &= \sigma^{2}\widetilde{\mathbf{C}}_{k,j,t} \boldsymbol{\Sigma}_k^{-1} \hat{\boldsymbol{\beta}}_{k0,t} + \widetilde{\mathbf{C}}_{k,j,t} \mathbf{X}_{k,j,t}^\top\mathbf{y}_{k,j,t} \\
			&= \boldsymbol{\beta}_{k,j} + \sigma^2 \widetilde{\mathbf{C}}_{k,j,t} \boldsymbol{\Sigma}_{k}^{-1}\left(\hat{\boldsymbol{\beta}}_{k0,t} - \boldsymbol{\beta}_{k,j} \right) + \widetilde{\mathbf{C}}_{k,j,t} \mathbf{X}_{k,j,t}^\top\boldsymbol{\epsilon}_{k,j,t},
		\end{aligned}
	\end{equation*}
	where the last equality is due to  $\widetilde{\mathbf{C}}_{k,j,t} = (\mathbf{X}_{k,j,t}^\top\mathbf{X}_{k,j,t} + \sigma^{2}\boldsymbol{\Sigma}_k^{-1})^{-1}$. 
\end{proof}

\begin{lemma}\label{XVX}
	We have
	$$
	\mathbf{X}_{k,j,t}^\top \mathbf{V}_{k,j,t}^{-1} \mathbf{X}_{k,j,t} = \boldsymbol{\Sigma}_{k}^{-1} - \sigma^2 \boldsymbol{\Sigma}_{k}^{-1} \widetilde{\mathbf{C}}_{k,j,t} \boldsymbol{\Sigma}_{k}^{-1}.
	$$
\end{lemma}
\begin{proof}
	From (\ref{efficiently compute V}), we have
	\begin{equation*}
		\begin{aligned}
			\mathbf{X}_{k,j,t}^\top\mathbf{V}_{k,j,t}^{-1}\mathbf{X}_{k,j,t} 
			= &\ \sigma^{-2}\mathbf{X}_{k,j,t}^\top\mathbf{X}_{k,j,t} - \sigma^{-2}\mathbf{X}_{k,j,t}^\top\mathbf{X}_{k,j,t}\widetilde{\mathbf{C}}_{k,j,t}\mathbf{X}_{k,j,t}^\top\mathbf{X}_{k,j,t} \\
			=&\ \boldsymbol{\Sigma}_k^{-1}\widetilde{\mathbf{C}}_{k,j,t}\mathbf{X}_{k,j,t}^\top\mathbf{X}_{k,j,t} \\
			=&\ \boldsymbol{\Sigma}_k^{-1}\widetilde{\mathbf{C}}_{k,j,t}\left( \mathbf{X}_{k,j,t}^\top\mathbf{X}_{k,j,t}+\sigma^{2}\boldsymbol{\Sigma}_k^{-1}- \sigma^{2}\boldsymbol{\Sigma}_k^{-1}\right) \\
			=&\ \boldsymbol{\Sigma}_{k}^{-1} - \sigma^2 \boldsymbol{\Sigma}_{k}^{-1} \widetilde{\mathbf{C}}_{k,j,t} \boldsymbol{\Sigma}_{k}^{-1},
		\end{aligned}
	\end{equation*}
	where the second equality is from $\widetilde{\mathbf{C}}_{k,j,t} = (\mathbf{X}_{k,j,t}^\top\mathbf{X}_{k,j,t} + \sigma^{2}\boldsymbol{\Sigma}_k^{-1})^{-1}$.
\end{proof}

\paragraph{Proof of Theorem \ref{inequality_beta_hat}}
By the Cauchy-Schwarz inequality, we have
\begin{align}\label{x(betahat-beta)}
	\left|  \mathbf{x}^\top
	\left(\hat{\boldsymbol{\beta}}_{k,j,t} -\boldsymbol{\beta}_{k,j} \right) \right|  
	\leq&\ \sigma^{-1}\|\mathbf{C}_{k,j,t}^{\frac{1}{2}}\mathbf{x}\|_2 \cdot \|\widetilde{\mathbf{C}}_{k,j,t}^{-\frac{1}{2}}\left( \hat{\boldsymbol{\beta}}_{k,j,t} - \boldsymbol{\beta}_{k,j}\right)\|_2 .
\end{align}
Now we derive the upper bound for $\|\widetilde{\mathbf{C}}_{k,j,t}^{-\frac{1}{2}}\left( \hat{\boldsymbol{\beta}}_{k,j,t} - \boldsymbol{\beta}_{k,j}\right)\|_2$. 
According to Lemma \ref{beta_hat-beta}, the estimation error is decomposed into the following two components
$$
\hat{\boldsymbol{\beta}}_{k,j,t} - \boldsymbol{\beta}_{k,j} = \sigma^2 \widetilde{\mathbf{C}}_{k,j,t} \boldsymbol{\Sigma}_{k}^{-1}\left(\hat{\boldsymbol{\beta}}_{k0,t} - \boldsymbol{\beta}_{k,j} \right)  + \widetilde{\mathbf{C}}_{k,j,t}\mathbf{X}_{k,j,t}\boldsymbol{\epsilon}_{k,j,t}.
$$
By employing the triangle inequality, $\|\widetilde{\mathbf{C}}_{k,j,t}^{-\frac{1}{2}}\left( \hat{\boldsymbol{\beta}}_{k,j,t} - \boldsymbol{\beta}_{k,j}\right)\|_2$ can be upper bounded as follows:
\begin{align}\label{eq_betahat_l2_upper}
	\|\widetilde{\mathbf{C}}_{k,j,t}^{-\frac{1}{2}}\left( \hat{\boldsymbol{\beta}}_{k,j,t} - \boldsymbol{\beta}_{k,j}\right)\|_2 
	\leq \|\sigma^2 	\widetilde{\mathbf{C}}_{k,j,t}^{\frac{1}{2}} \boldsymbol{\Sigma}_{k}^{-1}\left(\hat{\boldsymbol{\beta}}_{k0,t} - \boldsymbol{\beta}_{k,j} \right)\|_2 + \|\widetilde{\mathbf{C}}_{k,j,t}^{\frac{1}{2}} \mathbf{X}_{k,j,t}^\top\boldsymbol{\epsilon}_{k,j,t}\|_2.
\end{align}
We first analyze the first term on the right-hand side of (\ref{eq_betahat_l2_upper}). Through algebraic manipulation and noting that $\sigma^2\boldsymbol{\Sigma}_{k}^{-1} \widetilde{\mathbf{C}}_{k,j,t} \boldsymbol{\Sigma}_{k}^{-1}\preceq \boldsymbol{\Sigma}_{k}^{-1}$, we have 
\begin{align}\label{beta_k0-beta}
	\notag\|\sigma^2 \widetilde{\mathbf{C}}_{k,j,t}^{\frac{1}{2}} \boldsymbol{\Sigma}_{k}^{-1}\left(\hat{\boldsymbol{\beta}}_{k0,t} - \boldsymbol{\beta}_{k,j} \right)\|_2^2 
	&= \left(\hat{\boldsymbol{\beta}}_{k0,t} - \boldsymbol{\beta}_{k,j} \right)^\top \sigma^4 \boldsymbol{\Sigma}_{k}^{-1} \widetilde{\mathbf{C}}_{k,j,t} \boldsymbol{\Sigma}_{k}^{-1} \left(\hat{\boldsymbol{\beta}}_{k0,t} - \boldsymbol{\beta}_{k,j} \right)\\ \notag
	&\leq \left(\hat{\boldsymbol{\beta}}_{k0,t} - \boldsymbol{\beta}_{k,j} \right)^\top\sigma^2\boldsymbol{\Sigma}_{k}^{-1}\left(\hat{\boldsymbol{\beta}}_{k0,t} - \boldsymbol{\beta}_{k,j} \right)\\
	&\leq \sigma^2\lambda_1 \|\hat{\boldsymbol{\beta}}_{k0,t} - \boldsymbol{\beta}_{k,j}\|_2^2,
\end{align}
where the last equality is due to Assumption \ref{assumption}. 
Let $\widetilde{\mathbf{X}}_{k,j,t} = \mathbf{V}_{k,j,t}^{-\frac{1}{2}}\mathbf{X}_{k,j,t},\ \widetilde{\mathbf{y}}_{k,j,t} = \mathbf{V}_{k,j,t}^{-\frac{1}{2}} \mathbf{y}_{k,j,t},\ \widetilde{\boldsymbol{\epsilon}}_{k,j,t} = \mathbf{V}_{k,j,t}^{-\frac{1}{2}}\boldsymbol{\epsilon}_{k,j,t}$, and $\mathbf{M}=\sum_{j=1}^N  \widetilde{\mathbf{X}}_{k,j,t}^\top \widetilde{\mathbf{X}}_{k,j,t}
+ \lambda\mathbf{I}$. 
We obtain
\begin{align}\label{betahat_k0_decompose}
	\notag \hat{\boldsymbol{\beta}}_{k0,t}
	=&\ \mathbf{M}^{-1} \sum_{j=1}^N  \widetilde{\mathbf{X}}_{k,j,t}^\top (\widetilde{\mathbf{X}}_{k,j,t}\boldsymbol{\beta}_{k,j} + \widetilde{\boldsymbol{\epsilon}}_{k,j,t})\\ \notag
	=&\ \sum_{s=1}^N\mathbf{M}^{-1}\widetilde{\mathbf{X}}_{k,s,t}^\top \widetilde{\mathbf{X}}_{k,s,t}\boldsymbol{\beta}_{k,s} + \mathbf{M}^{-1} \sum_{j=1}^N  \widetilde{\mathbf{X}}_{k,j,t}^\top\widetilde{\boldsymbol{\epsilon}}_{k,j,t},
\end{align}
It follows
$$
\begin{aligned}
	\|\hat{\boldsymbol{\beta}}_{k0,t} - \boldsymbol{\beta}_{k,j}\|_2
	\leq&\ Nc_{\beta}	+
	\frac{1}{\sqrt{\lambda}}\left\|\mathbf{M}^{-\frac{1}{2}} \sum_{j=1}^N  \widetilde{\mathbf{X}}_{k,j,t}^\top\widetilde{\boldsymbol{\epsilon}}_{k,j,t}\right\|_2,
\end{aligned}
$$
where the inequality arises from the fact that, for all $s\in[N]$, the eigenvalues of $\widetilde{\mathbf{X}}_{k,s,t}^\top \widetilde{\mathbf{X}}_{k,s,t}$ are smaller then those of $\sum_{j=1}^N\widetilde{\mathbf{X}}_{k,j,t}^\top \widetilde{\mathbf{X}}_{k,j,t}$. 
By applying Lemma \ref{abbasi_inequality}, we obtain the following inequality, which holds with probability at least $1-\delta/2$:
\begin{equation}\label{betahat_l2_component1_ub}
	\|\hat{\boldsymbol{\beta}}_{k0,t} - \boldsymbol{\beta}_{k,j}\|_2
	\leq Nc_{\beta} + \sqrt{ d /\lambda \log \left( \frac{\lambda + t c_x}{\lambda\delta/2} \right)}.
\end{equation}

We next analyze the second term on the right-hand side of (\ref{eq_betahat_l2_upper}) using Lemma \ref{abbasi_inequality}.  
Noting that $\lambda_d \leq \lambda(\boldsymbol{\Sigma}_{k}^{-1}) \leq \lambda_1$, we obtain the following inequality, which holds with probability at least $1-\delta/2$: 
\begin{equation}\label{betahat_l2_component2_ub}
	\|\widetilde{\mathbf{C}}_{k,j,t}^{\frac{1}{2}} \mathbf{X}_{k,j,t}^\top\boldsymbol{\epsilon}_{k,j,t}\|_2 \leq \sqrt{\sigma^2 d \log \left( \frac{\sigma^2\lambda_1 + t c_x}{\sigma^2\lambda_{d}\delta/2} \right)}.
\end{equation}
By combining (\ref{beta_k0-beta}), (\ref{betahat_l2_component1_ub}), and (\ref{betahat_l2_component2_ub}), and substituting the resulting expressions into (\ref{eq_betahat_l2_upper}), the following holds with probability at least $1-\delta$:
$$
\begin{aligned}
	\|\widetilde{\mathbf{C}}_{k,j,t}^{-\frac{1}{2}}\left( 	\hat{\boldsymbol{\beta}}_{k,j,t} - \boldsymbol{\beta}_{k,j}\right)\|_2 
	\leq &\ \sqrt{ \sigma^2 d 	\lambda_1/\lambda \log \left( \frac{\lambda + t c_x}{\lambda\delta/2} \right)} 
	+ \sqrt{\sigma^2 d \log \left( \frac{\sigma^2\lambda_1 + t c_x}{\sigma^2\lambda_{d}\delta/2} \right)} +\sigma Nc_{\beta}\sqrt{\lambda_1}\\
	\leq &\ 2\sqrt{\sigma^2 d \max\{\lambda_1/\lambda, 1\} \log \left(\frac{\max\{\lambda, \sigma^2\lambda_1\} + t c_x}{\sqrt{\lambda\lambda_d \sigma^2}\delta/2} \right) } +\sigma Nc_{\beta}\sqrt{\lambda_1} .
\end{aligned}	
$$
Substituting it into (\ref{x(betahat-beta)}) completes the proof.

\section{Proof of Theorem \ref{regret_bound_ebmUCB}}\label{appendix_thm3}

\paragraph{Proof}
The U-value for arm $k$ of bandit $j$ at time step $t$ is defined as 
$$
U_{k,j,t} = \mathbf{x}_t^\top \hat{\boldsymbol{\beta}}_{k,j,t} + \alpha_{t}(\delta)\|\mathbf{x}_t\|_{\mathbf{C}_{k,j,t}}.
$$
Note that at each time step, the \textsc{ebmUCB} algorithm selects the arm with the highest U-value. 
We define the event $\xi_t$ as 
$$
\xi_t = \left\lbrace \forall k \in [K], \forall j \in [N]:  \left|  \mathbf{x}_t^\top(\hat{\boldsymbol{\beta}}_{k,j,t} - \boldsymbol{\beta}_{k,j}) \right|  \leq  \alpha_{t}(\delta)\|\mathbf{x}_t\|_{\mathbf{C}_{k,j,t}} \right\rbrace.
$$
The regret $R_n$ can be decomposed as 
\begin{align}\label{ebmUCB_regret}
	\notag R_{n} 
	= & \sum_{t=1}^n \mathbb{E}\left[ \mathbb{I}(\xi_t) \left(\mathbf{x}_t^\top \boldsymbol{\beta}_{\pi^*_{t}, Z_t} 
	- \mathbf{x}_t^\top \boldsymbol{\beta}_{\pi_{t}, Z_t}  \right) \right]
	+ \sum_{t=1}^n \mathbb{E}\left[ \mathbb{I}(\bar{\xi}_t) \left(\mathbf{x}_t^\top \boldsymbol{\beta}_{\pi^*_{t}, Z_t} - \mathbf{x}_t^\top \boldsymbol{\beta}_{\pi_{t}, Z_t}  \right) \right]\\ 
	\leq &\sum_{t=1}^n \mathbb{E}\left[  \left(\mathbf{x}_t^\top \boldsymbol{\beta}_{\pi^*_{t}, Z_t} - \mathbf{x}_t^\top \boldsymbol{\beta}_{\pi_{t}, Z_t}  \right) \mid \xi_t\right] 
	+ 2 \sqrt{d c_x}c_{\beta}\sum_{t=1}^n \mathbb{P}(\bar{\xi}_t).
\end{align}
In the right-hand side of (\ref{ebmUCB_regret}), the upper bound of $R_{n}$ consists of two terms. The first term represents the gap between the optimal arm and the chosen arm, conditioned on the event $\xi_t$. The second term corresponds to the probability that $\xi_t$ does not hold.
We first focus on the first term in the decomposition of (\ref{ebmUCB_regret}), with each subterm in this summation given as follows:
$$
\begin{aligned}
	\mathbb{E}\left[  \left(\mathbf{x}_t^\top \boldsymbol{\beta}_{\pi^*_{t}, Z_t} - \mathbf{x}_t^\top \boldsymbol{\beta}_{\pi_{t}, Z_t}  \right) \mid \xi_t \right] 
	\leq&\ \mathbb{E}\left[ \left(\mathbf{x}_t^\top \boldsymbol{\beta}_{\pi^*_{t}, Z_t} - U_{\pi^*_{t}, Z_t, t} + U_{\pi_{t}, Z_t, t} - \mathbf{x}_t^\top \boldsymbol{\beta}_{\pi_{t}, Z_t}  \right) \mid \xi_t\right]\\ \notag
	\leq&\  2\mathbb{E}\left[ \alpha_{t}(\delta) \|\mathbf{x}_t\|_{\mathbf{C}_{\pi_{t}, Z_t, t}} \mid \xi_t \right] \\
	\leq&\  2\alpha_{n}(\delta)\mathbb{E}\left[ \sqrt{ \mathbf{x}_t^\top \mathbf{C}_{\pi_{t}, Z_t, t} \mathbf{x}_t }\mid \xi_t \right],
\end{aligned}
$$
where the first inequality follows from  the fact that $U_{\pi_{t},Z_t,t} \geq U_{\pi^*_{t},Z_t,t}$ when selecting arm $\pi_{t}$, and the second inequality holds because, for any arm $k$ conditioned on $\xi_t$, $0 \leq U_{k,j,t} - \mathbf{x}_t^\top \boldsymbol{\beta}_{k, j} \leq 2\alpha_{t}(\delta) \|\mathbf{x}_t\|_{\mathbf{C}_{k,j,t}}$.   
As a result, we sum over all possible values of $t$,
\begin{align*}
	\sum_{t=1}^n \mathbb{E}\left[  \left(\mathbf{x}_t^\top \boldsymbol{\beta}_{\pi^*_{t}, Z_t} - \mathbf{x}_t^\top \boldsymbol{\beta}_{\pi_{t}, Z_t}  \right) \mid \xi_t \right] 
	\leq&\ 2\alpha_{n}(\delta) \sum_{t=1}^n \mathbb{E}\left[ \sqrt{ \mathbf{x}_t^\top \mathbf{C}_{\pi_{t}, Z_t, t} \mathbf{x}_t }\mid \xi_t \right] \\
	\leq&\ 2\alpha_{n}(\delta) \sqrt{n\ \mathbb{E}\left[  \sum_{t=1}^n \mathbf{x}_t^\top \mathbf{C}_{\pi_{t}, Z_t, t} \mathbf{x}_t \mid \xi_t \right]},
\end{align*}
where the last inequality follows from the Cauchy–Schwarz inequality and the concavity of the square root, which implies that, for non-negative random variables $X_i$, $\sum_{i=1}^n \mathbb{E}(\sqrt{X_i}) \leq \sqrt{n \mathbb{E}(\sum_{i=1}^n X_i)}$. 

From (\ref{estimation_beta_{k,j}}), $\mathbf{x}_t^\top \mathbf{C}_{k,j,t} \mathbf{x}_t$ can be decomposed into two terms as follows
\begin{equation}\label{decompose_xCx}
	\mathbf{x}_t^\top \mathbf{C}_{k,j,t} \mathbf{x}_t = \sigma^{2}\mathbf{x}_t^\top\widetilde{\mathbf{C}}_{k,j,t}\mathbf{x}_t + \sigma^{4}
	\mathbf{x}_t^\top\widetilde{\mathbf{C}}_{k,j,t} \boldsymbol{\Sigma}_k^{-1} \boldsymbol{\Phi}_{k0,t}
	\boldsymbol{\Sigma}_k^{-1}\widetilde{\mathbf{C}}_{k,j,t}\mathbf{x}_t.
\end{equation}
The detailed derivation of the upper bound for the summation of the first term in (\ref{decompose_xCx}) is provided in Lemma \ref{variance_sum_1}, which yields
\begin{equation}\label{ebmUCB_var1_bound}
	\sigma^2 \sum_{t=1}^n \mathbf{x}_t^\top\widetilde{\mathbf{C}}_{\pi_{t}, Z_t, t} \mathbf{x}_t \leq 
	c_1 d K \sum_{j=1}^N \log \left(1+ c_2 n_j\right).
\end{equation}
For the second term in (\ref{decompose_xCx}), Lemma \ref{variance_sum_2} implies
\begin{align}\label{ebmUCB_var2_bound}
	\sum_{t=1}^n  \sigma^{4}
	\mathbf{x}_t^\top\widetilde{\mathbf{C}}_{\pi_{t}, Z_t, t} 	\boldsymbol{\Sigma}_{\pi_{t}}^{-1} \boldsymbol{\Phi}_{\pi_{t}0,t}
	\boldsymbol{\Sigma}_{\pi_{t}}^{-1}\widetilde{\mathbf{C}}_{\pi_{t}, Z_t, t}\mathbf{x}_t  
	\leq c_3 d K \log \left(1+ c_4 N\right).	
\end{align}   
Next, we bound the probability that $\xi_t$ does not hold,
\begin{equation}\label{ebmUCB_bar_xi_bound}
	\mathbb{P}(\bar{\xi}_t) \leq KN \delta.
\end{equation}
Combining (\ref{ebmUCB_var1_bound}), (\ref{ebmUCB_var2_bound}) and (\ref{ebmUCB_bar_xi_bound}), we can derive the upper bound of (\ref{ebmUCB_regret}): 
$$
\begin{aligned}
	R_n 
	\leq 2\alpha_{n}(\delta) \sqrt{nd K \left[ c_1 \sum_{j=1}^N\log (1 + c_2 n_j)+c_3 \log (1+ c_4 N) \right] } + 2 \sqrt{d c_x}c_{\beta}KNn \delta. 
\end{aligned}
$$

\section{Proof of Theorem \ref{regret_bound_ebmTS}}\label{appendix_thm2}

\paragraph{Proof}
We define two following events $\xi^{(1)}_t$ and $\xi^{(2)}_t$ as: 
$$
\begin{aligned}
	\xi^{(1)}_t &= \left\lbrace \forall k \in [K], \forall j \in [N]:  \left|  \mathbf{x}_t^\top(\hat{\boldsymbol{\beta}}_{k,j,t} - \boldsymbol{\beta}_{k,j})  \right|  \leq  \alpha_{t}(\delta)\|\mathbf{x}_t\|_{\mathbf{C}_{k,j,t}} \right\rbrace \\
	\xi^{(2)}_t &= \left\lbrace \forall k \in [K], \forall j \in [N]:  \left|  \mathbf{x}^\top (\breve{\boldsymbol{\beta}}_{k,j,t} - \hat{\boldsymbol{\beta}}_{k,j,t}) \right|  \leq  \alpha_{t}(\delta)\gamma_t \|\mathbf{x}\|_{\mathbf{C}_{k,j,t}} \right\rbrace,
\end{aligned}
$$
where $\gamma_t= \sqrt{2d\log dt^2}$.
The event $\xi^{(2)}_t$ is introduced because, as shown in Lemma \ref{sample-error-bound}, the error between the sampled variable $\mathbf{x}^\top \breve{\boldsymbol{\beta}}_{k,j,t}$ and its mean $\mathbf{x}^\top \hat{\boldsymbol{\beta}}_{k,j,t}$ is bounded by $\alpha_{t}(\delta) \gamma_t \|\mathbf{x}\|_{\mathbf{C}_{k,j,t}}$ with probability of at least $1-1/t^2$ for all $k\in[K]$ and $j\in[N]$.
The cumulative regret can be decomposed into
\begin{align*}
	R_{n} \leq & \sum_{t=1}^n \mathbb{E}\left[  \left(\mathbf{x}_t^\top \boldsymbol{\beta}_{\pi^*_{t}, Z_t} - \mathbf{x}_t^\top \boldsymbol{\beta}_{\pi_t, Z_t}  \right) \mid \xi^{(1)}_t, \xi^{(2)}_t \right] 
	+ 2 \sqrt{d c_x}c_{\beta}\sum_{t=1}^n \left(\mathbb{P} (\bar{\xi}^{(1)}_t) + \mathbb{P} (\bar{\xi}^{(2)}_t)\right) .
\end{align*}

Denote the set of saturated arms \citep{agrawal2013thompson} as,
$$
\mathcal{S}(t)=\left\{k \in [K]: \Delta_{k,t} > (\gamma_t + 1) \alpha_{t}(\delta)\|\boldsymbol{\mathbf{x}}_t\|_{\mathbf{C}_{k, Z_t, t}}\right\},
$$
where $\Delta_{k,t} = \mathbf{x}_t^\top \boldsymbol{\beta}_{\pi^*_{t}, Z_t} - \mathbf{x}_t^\top \boldsymbol{\beta}_{k, Z_t}$. Let $\pi_t^{\dagger}$ represent the unsaturated arm with the smallest $\|\boldsymbol{\mathbf{x}}_t\|_{\mathbf{C}_{k, Z_t, t}}$ norm, i.e.,
$$
\pi_t^{\dagger} = \underset{k \notin \mathcal{S}(t)}{\arg \min }\|\boldsymbol{\mathbf{x}}_t\|_{\mathbf{C}_{k, Z_t, t}}.
$$
The existence of an unsaturated arm  $\pi_t^{\dagger} $ is guaranteed since  $\pi^*_{t} \notin \mathcal{S}(t)$. When both  $\xi^{(1)}_t$ and  $\xi^{(2)}_t$ hold, the suboptimality can be expressed as follows:
$$
\begin{aligned}
	\Delta_{\pi_{t},t} &= \Delta_{\pi_t^{\dagger}, t} + \mathbf{x}_t^\top \boldsymbol{\beta}_{\pi_t^{\dagger}, Z_t} - \mathbf{x}_t^\top \boldsymbol{\beta}_{\pi_{t}, Z_t} \\
	&\leq  \Delta_{\pi_t^{\dagger}, t} + \left(\mathbf{x}_t^\top \breve{\boldsymbol{\beta}}_{\pi_t^{\dagger}, Z_t, t} + (\gamma_t + 1) \alpha_{t}(\delta)\|\boldsymbol{\mathbf{x}}_t\|_{\mathbf{C}_{\pi_t^{\dagger}, Z_t, t}} \right) 
	- \left(\mathbf{x}_t^\top \breve{\boldsymbol{\beta}}_{\pi_{t}, Z_t, t} - (\gamma_t + 1) \alpha_{t}(\delta)\|\boldsymbol{\mathbf{x}}_t\|_{\mathbf{C}_{\pi_{t}, Z_t, t}} \right),
\end{aligned}
$$
Since we select arm $\pi_{t}$ at time step $t$, i.e., $\mathbf{x}_t^\top \breve{\boldsymbol{\beta}}_{\pi_{t}, Z_t, t} > \mathbf{x}_t^\top \breve{\boldsymbol{\beta}}_{\pi_t^{\dagger}, Z_t, t}$, it follows
$$
\begin{aligned}
	\Delta_{\pi_{t},t} \leq &\ \Delta_{\pi_t^{\dagger}, t} + (\gamma_t + 1) \alpha_{t}(\delta)\|\boldsymbol{\mathbf{x}}_t\|_{\mathbf{C}_{\pi_t^{\dagger}, Z_t, t}}
	+  (\gamma_t + 1) \alpha_{t}(\delta)\|\boldsymbol{\mathbf{x}}_t\|_{\mathbf{C}_{\pi_{t}, Z_t, t}}\\
	\leq&\ (\gamma_t + 1) \alpha_{t}(\delta)\left(
	2\|\boldsymbol{\mathbf{x}}_t\|_{\mathbf{C}_{\pi_t^{\dagger}, Z_t, t}} 
	+ \|\boldsymbol{\mathbf{x}}_t\|_{\mathbf{C}_{\pi_{t}, Z_t, t}}
	\right),
\end{aligned}
$$
where the second inequality results from  $\pi_t^{\dagger} \notin \mathcal{S}(t)$. Note that  $\|\boldsymbol{\mathbf{x}}_t\|_{\mathbf{C}_{\pi_{t}, Z_t, t}} \geq \|\boldsymbol{\mathbf{x}}_t\|_{\mathbf{C}_{\pi_t^{\dagger}, Z_t, t}} $  with constant probability, i.e.,
$$
\begin{aligned}	
	\mathbb{E}\left[\|\boldsymbol{\mathbf{x}}_t\|_{\mathbf{C}_{\pi_{t}, Z_t, t}} \mid \xi^{(1)}_t, \xi^{(2)}_t\right] 
	\geq&\  \mathbb{E}\left[		\|\boldsymbol{\mathbf{x}}_t\|_{\mathbf{C}_{\pi_{t}, Z_t, t}}
	\mid \pi_{t} \notin \mathcal{S}(t), \xi^{(1)}_t, \xi^{(2)}_t\right] \cdot \mathbb{P}\left(\pi_{t}\notin \mathcal{S}(t) \mid \xi^{(1)}_t, \xi^{(2)}_t \right) \\
	\geq&\ \left(\frac{1}{2 \sqrt{2\pi e}}-\frac{1}{t^{2}}\right) \mathbb{E}\left[
	\|\boldsymbol{\mathbf{x}}_t\|_{\mathbf{C}_{\pi_t^{\dagger}, Z_t, t}} \mid \xi^{(1)}_t, \xi^{(2)}_t \right],
\end{aligned}
$$
The detailed proof for the lower bound of $\mathbb{P}\left(\pi_{t}\notin \mathcal{S}(t) \mid \xi^{(1)}_t, \xi^{(2)}_t \right)$ is provided in Lemma \ref{xt-unsaturated}, and the last inequality follows from the definition of  $\pi_t^{\dagger}$ as the unsaturated arm with the smallest  $\|\boldsymbol{\mathbf{x}}_t\|_{\mathbf{C}_{k, Z_t, t}}$. As a result, 
$$
\begin{aligned}
	\mathbb{E}\left[\Delta_{\pi_{t},t} \mid \xi^{(1)}_t, \xi^{(2)}_t\right] 
	\leq &\ (\gamma_t+1) \alpha_{t}(\delta)\mathbb{E}\left[
	2\|\boldsymbol{\mathbf{x}}_t\|_{\mathbf{C}_{\pi_t^{\dagger}, Z_t, t}} 
	+ \|\boldsymbol{\mathbf{x}}_t\|_{\mathbf{C}_{\pi_{t}, Z_t, t}}
	\mid \xi^{(1)}_t, \xi^{(2)}_t \right]
	+ 2 \sqrt{d c_x} c_{\beta}NK\left( \delta + \frac{1}{t^{2}} \right)  \\
	\leq&\ \left(\frac{2}{\frac{1}{2  \sqrt{2\pi e}}-\frac{1}{t^{2}}}+1\right)
	( \gamma_t+1) \alpha_{t}(\delta) \mathbb{E}\left[\|\boldsymbol{\mathbf{x}}_t\|_{\mathbf{C}_{\pi_{t}, Z_t, t}} \mid \xi^{(1)}_t, \xi^{(2)}_t \right] 
	+ 2 \sqrt{d c_x} c_{\beta}NK\left( \delta + \frac{1}{t^{2}} \right).
\end{aligned}
$$
Let $C=\max _{t \geq 1} \frac{2}{\left|\frac{1}{2  \sqrt{2\pi e}}-\frac{1}{t^{2}}\right|}+1$, then sum all $t$, we have
$$
\begin{aligned}
	R_n 
	\leq C (\gamma_n+1)\alpha_{n}(\delta)\sqrt{n\ \mathbb{E}\left[  \sum_{t=1}^n \|\mathbf{x}_t\|_{\mathbf{C}_{\pi_{t}, Z_t, t}}^2 \mid \xi^{(1)}_t, \xi^{(2)}_t \right]} 
	+ 2 \sqrt{d c_x} c_{\beta}KN(n\delta + \pi^2/6).
\end{aligned}
$$
Combining the upper bound in Lemma \ref{variance_sum_1} and Lemma \ref{variance_sum_2}, we obtain the final result.

\begin{lemma}\label{sample-error-bound}
	The sampled variable $\mathbf{x}^\top\breve{\boldsymbol{\beta}}_{k,j,t}$ in the algorithm \textsc{ebmTS} satisfies the following bound on  the sampling error: 
	$$
	\left|  \mathbf{x}^\top
	\left(\breve{\boldsymbol{\beta}}_{k,j,t} - \hat{\boldsymbol{\beta}}_{k,j,t} \right) \right|  \leq  \alpha_{t}(\delta) \sqrt{2d\log d/\delta^{\prime}} \|\mathbf{x}\|_{\mathbf{C}_{k,j,t}},
	$$
	with probability at least $1 - \delta^{\prime}$.
\end{lemma}
\begin{proof}
	Note that $ \breve{\boldsymbol{\beta}}_{k,j,t}\sim \mathcal{N}\left(\hat{\boldsymbol{\beta}}_{k,j,t}, \alpha_{t}^2(\delta) \mathbf{C}_{k,j,t}\right)$, which implies  $\mathbf{C}_{k,j,t}^{-\frac{1}{2}}\left( \breve{\boldsymbol{\beta}}_{k,j,t} - \hat{\boldsymbol{\beta}}_{k,j,t}\right) / \alpha_{t}(\delta) \sim \mathcal{N}(\mathbf{0},\mathbf{I}_d)$. Applying the anti-concentration inequality from Lemma \ref{concentration_gaussian} for the standard normal distribution, we have that, for every  $t \in[n]$, the following holds with probability at least $1-\delta^{\prime}$: 
	$$
	\begin{aligned}
		\left|\mathbf{x}^\top \left(\breve{\boldsymbol{\beta}}_{k,j,t} - \hat{\boldsymbol{\beta}}_{k,j,t}\right)\right|
		& =\left|\boldsymbol{\mathbf{x}}^{\top} \mathbf{C}_{k,j,t}^{\frac{1}{2}} \mathbf{C}_{k,j,t}^{-\frac{1}{2}} \left(\breve{\boldsymbol{\beta}}_{k,j,t} - \hat{\boldsymbol{\beta}}_{k,j,t} \right)\right| \\
		& \leq  \left\|\mathbf{C}_{k,j,t}^{-\frac{1}{2}} \left(\breve{\boldsymbol{\beta}}_{k,j,t} - \hat{\boldsymbol{\beta}}_{k,j,t}\right)\right\|_{2} \|\boldsymbol{\mathbf{x}}\|_{\mathbf{C}_{k,j,t}} \\
		& \leq \alpha_{t}(\delta) \sqrt{2d\log d/\delta^{\prime}} \|\boldsymbol{\mathbf{x}}\|_{\mathbf{C}_{k,j,t}},
	\end{aligned}
	$$
	where the first inequality follows from the Cauchy-Schwarz inequality. 
	We now provide a detailed explanation of the second inequality. Let $\mathbf{C}_{k,j,t}^{-\frac{1}{2}}\left( \breve{\boldsymbol{\beta}}_{k,j,t} - \hat{\boldsymbol{\beta}}_{k,j,t}\right) / \alpha_{t}(\delta) =: (u_1,u_2,\ldots,u_d)$, where each $u_i$ is a standardized random variable with unit norm, and the $u_i$ are independent of one another. We have
	$$
	\begin{aligned}
		\mathbb{P}\left( \left\|\frac{\mathbf{C}_{k,j,t}^{-\frac{1}{2}} \left(\breve{\boldsymbol{\beta}}_{k,j,t} - \hat{\boldsymbol{\beta}}_{k,j,t}\right)}{\alpha_{t}(\delta) }\right\|_{2} \geq \sqrt{2d\log d/\delta^{\prime}}	\right) 
		=&\ \mathbb{P}\left(\sqrt{u_1^2+\ldots+u_d^2}\geq \sqrt{2d\log d/\delta^{\prime}}\right) \\
		\leq&\ \mathbb{P}\left(\exists i\in[d], u_i\geq \sqrt{2\log d/\delta^{\prime}}\right) \\
		\leq&\ d\,\mathbb{P}\left(u_1\geq \sqrt{2\log d/\delta^{\prime}}\right)
		\leq \delta^{\prime}.
	\end{aligned}
	$$	
\end{proof}

\begin{lemma}\label{optimistic}
	For any  $t \geq 1$, conditional on $\xi^{(1)}_t$ holding, there exists a constant probability that the sampled variable $\boldsymbol{\mathbf{x}}_{t}^{\top} \breve{\boldsymbol{\beta}}_{\pi^*_{t}, Z_t, t}$ of the optimal arm serves as an upper confidence bound, i.e.,
	$$
	\mathbb{P}\left(\boldsymbol{\mathbf{x}}_{t}^{\top} \breve{\boldsymbol{\beta}}_{\pi^*_{t}, Z_t, t} > 
	\boldsymbol{\mathbf{x}}_{t}^{\top} \boldsymbol{\beta}_{\pi^*_{t}, Z_t, t}
	\mid \xi^{(1)}_t \right) \geq \frac{1}{2\sqrt{2\pi e}}.
	$$
\end{lemma}
\begin{proof}
	In the \textsc{ebmTS} algorithm, $\breve{\boldsymbol{\beta}}_{k, j, t}$ is sampled from a normal distribution with mean $\hat{\boldsymbol{\beta}}_{k, j, t}$.  
	Using this, the probability that the sampled variable $\boldsymbol{\mathbf{x}}_{t}^{\top} \breve{\boldsymbol{\beta}}_{\pi^*_{t}, Z_t, t}$  serves as an upper bound for the true reward is lower-bounded by 
	$$
	\begin{aligned}
		\mathbb{P}\left(\boldsymbol{\mathbf{x}}_{t}^{\top} \breve{\boldsymbol{\beta}}_{\pi^*_{t}, Z_t, t} > 
		\boldsymbol{\mathbf{x}}_{t}^{\top} \boldsymbol{\beta}_{\pi^*_{t}, Z_t, t} \mid \xi^{(1)}_t \right) 
		&= \mathbb{P}\left( \frac{\boldsymbol{\mathbf{x}}_{t}^{\top} (\breve{\boldsymbol{\beta}}_{\pi^*_{t}, Z_t, t} 
			- \hat{\boldsymbol{\beta}}_{\pi^*_{t}, Z_t, t})}{\alpha_{t}(\delta)\|\boldsymbol{\mathbf{x}}_t\|_{\mathbf{C}_{\pi^*_{t}, Z_t, t}}} 
		> \frac{\boldsymbol{\mathbf{x}}_{t}^{\top} 	(\boldsymbol{\beta}_{\pi^*_{t}, Z_t, t} -\hat{\boldsymbol{\beta}}_{\pi^*_{t}, Z_t, t})}{\alpha_{t}(\delta)\|\boldsymbol{\mathbf{x}}_t\|_{\mathbf{C}_{\pi^*_{t}, Z_t, t}}}
		\mid \xi^{(1)}_t \right) \\
		&\geq  \mathbb{P}\left( \frac{\boldsymbol{\mathbf{x}}_{t}^{\top} \breve{\boldsymbol{\beta}}_{\pi^*_{t}, Z_t, t} 
			- \boldsymbol{\mathbf{x}}_{t}^{\top} \hat{\boldsymbol{\beta}}_{\pi^*_{t}, Z_t, t}}{\alpha_{t}(\delta)\|\boldsymbol{\mathbf{x}}_t\|_{\mathbf{C}_{\pi^*_{t}, Z_t, t}}} 
		> 1 \mid \xi^{(1)}_t \right) 
		\geq \frac{1}{2 \sqrt{2\pi e}},
	\end{aligned}
	$$
	Here, the first inequality follows from the fact that, under $\xi^{(1)}_t$, 
	$$
	\begin{aligned}
		\frac{\boldsymbol{\mathbf{x}}_{t}^{\top} 	\boldsymbol{\beta}_{\pi^*_{t}, Z_t, t} 
			- \boldsymbol{\mathbf{x}}_{t}^{\top} 	\hat{\boldsymbol{\beta}}_{\pi^*_{t}, Z_t, t}}{\alpha_{t}(\delta)\|\boldsymbol{\mathbf{x}}_t\|_{\mathbf{C}_{\pi^*_{t}, Z_t, t}}}  
		\leq \frac{\alpha_{t}(\delta)\|\boldsymbol{\mathbf{x}}_t\|_{\mathbf{C}_{\pi^*_{t}, Z_t, t}}}{\alpha_{t}(\delta)\|\boldsymbol{\mathbf{x}}_t\|_{\mathbf{C}_{\pi^*_{t}, Z_t, t}}}
		=1. 
	\end{aligned}
	$$
	The second inequality follows from the anti-concentration inequality for the standard normal distribution (see Lemma \ref{concentration_gaussian}) and from the fact that
	$$
	\frac{\boldsymbol{\mathbf{x}}_{t}^{\top} \breve{\boldsymbol{\beta}}_{\pi^*_{t}, Z_t, t} 
		- \boldsymbol{\mathbf{x}}_{t}^{\top} \hat{\boldsymbol{\beta}}_{\pi^*_{t}, Z_t, t}}{\alpha_{t}(\delta)\|\boldsymbol{\mathbf{x}}_t\|_{\mathbf{C}_{\pi^*_{t}, Z_t, t}}} \sim \mathcal{N}(0,1).
	$$
\end{proof}

\begin{lemma}\label{xt-unsaturated}
	For any  $t \geq 1$, conditional on $\xi^{(1)}_t$ and $\xi^{(2)}_t$ holding, there exists a constant probability that the chosen arm  $\pi_{t}$ is not a saturated arm, i.e.,
	$$
	\mathbb{P}\left(\pi_{t} \notin \mathcal{S}(t) \mid \xi^{(1)}_t, \xi^{(2)}_t\right) \geq \frac{1}{2 \sqrt{2\pi e}}-\frac{1}{t^{2}} .
	$$
\end{lemma}

\begin{proof}
	Recall the definition of the set of saturated arms:
	$$
	\mathcal{S}(t)=\left\{k \in [K]: \Delta_{k,t} > (\gamma_t + 1) \alpha_{t}(\delta)\|\boldsymbol{\mathbf{x}}_t\|_{\mathbf{C}_{k, Z_t, t}}\right\},
	$$
	and note that the selected arm  $\pi_{t} = \arg \max _{k \in [K]} \mathbf{x}_t^\top \breve{\boldsymbol{\beta}}_{k, Z_t, t}$  is chosen to maximize the sampled estimated reward. Consequently,  $\pi_{t} \notin \mathcal{S}(t)$  if  $\boldsymbol{\mathbf{x}}_{t}^{\top} \breve{\boldsymbol{\beta}}_{\pi^*_{t}, Z_t, t} > \boldsymbol{\mathbf{x}}^{\top} \breve{\boldsymbol{\beta}}_{k, Z_t, t}$  for all saturated arms  $k \in \mathcal{S}(t)$, meaning that the sampled estimated reward for the optimal arm exceeds those of all saturated arms. 
	This statement can be easily verified via the contrapositive: if $\pi_t \in \mathcal{S}(t)$, then there exists at least one arm in $\mathcal{S}(t)$ --- specifically $\pi_t$ itself --- such that 
	$$
	\boldsymbol{\mathbf{x}}_{t}^{\top} \breve{\boldsymbol{\beta}}_{\pi^*_{t}, Z_t, t} \leq \boldsymbol{\mathbf{x}}^{\top} \breve{\boldsymbol{\beta}}_{\pi_t, Z_t, t}.
	$$ 
	It follows
	\begin{align}\label{eqn:6-1}
		\mathbb{P}\left(\pi_{t} \notin \mathcal{S}(t) \mid \xi^{(1)}_t, \xi^{(2)}_t \right) 
		\geq&\ \mathbb{P}\left(\boldsymbol{\mathbf{x}}_{t}^{\top} \breve{\boldsymbol{\beta}}_{\pi^*_{t}, Z_t, t} > \boldsymbol{\mathbf{x}}_t^{\top} \breve{\boldsymbol{\beta}}_{k, Z_t, t}, \forall k \in \mathcal{S}(t) \mid \xi^{(1)}_t, \xi^{(2)}_t \right) .
	\end{align}
	When both events $\xi^{(1)}_t$  and  $\xi^{(2)}_t$  hold, for $k \in \mathcal{S}(t)$, we have
	$$
	\begin{aligned}
		\boldsymbol{\mathbf{x}}_{t}^{\top} \breve{\boldsymbol{\beta}}_{k, Z_t} 
		&\leq \boldsymbol{\mathbf{x}}_{t}^{\top} \boldsymbol{\beta}_{k, Z_t} + (\gamma_t + 1) \alpha_{t}(\delta)\|\boldsymbol{\mathbf{x}}_t\|_{\mathbf{C}_{k, Z_t, t}}\\
		&< \boldsymbol{\mathbf{x}}_{t}^{\top} \boldsymbol{\beta}_{k, Z_t} + \Delta_{k,t}
		=  \boldsymbol{\mathbf{x}}_{t}^{\top} \boldsymbol{\beta}_{\pi^*_{t}, Z_t}.
	\end{aligned}
	$$
	It follows
	\begin{align}\label{eqn:6-2}
		\mathbb{P}\left(\boldsymbol{\mathbf{x}}_{t}^{\top} \breve{\boldsymbol{\beta}}_{\pi^*_{t}, Z_t, t} > \boldsymbol{\mathbf{x}}_t^{\top} \breve{\boldsymbol{\beta}}_{k, Z_t, t}, \forall k \in \mathcal{S}(t) \mid \xi^{(1)}_t, \xi^{(2)}_t \right) 
		\geq&\ \mathbb{P}\left(\boldsymbol{\mathbf{x}}_{t}^{\top} \breve{\boldsymbol{\beta}}_{\pi^*_{t}, Z_t, t} > 
		\boldsymbol{\mathbf{x}}_{t}^{\top} \boldsymbol{\beta}_{\pi^*_{t}, Z_t, t} \mid \xi^{(1)}_t, \xi^{(2)}_t\right) \notag \\
		\geq&\ \mathbb{P}\left(\boldsymbol{\mathbf{x}}_{t}^{\top} \breve{\boldsymbol{\beta}}_{\pi^*_{t}, Z_t, t} > 
		\boldsymbol{\mathbf{x}}_{t}^{\top} \boldsymbol{\beta}_{\pi^*_{t}, Z_t, t} \mid \xi^{(1)}_t\right) - \mathbb{P}\left(\overline{\xi^{(2)}_t}\right),
	\end{align}	
	where the last inequality follows from 
	$$
	\begin{aligned}
		\mathbb{P}\left(\boldsymbol{\mathbf{x}}_{t}^{\top} \breve{\boldsymbol{\beta}}_{\pi^*_{t}, Z_t, t} > 
		\boldsymbol{\mathbf{x}}_{t}^{\top} \boldsymbol{\beta}_{\pi^*_{t}, Z_t, t} \mid \xi^{(1)}_t\right) 
		=&\  \mathbb{P}\left(\boldsymbol{\mathbf{x}}_{t}^{\top} \breve{\boldsymbol{\beta}}_{\pi^*_{t}, Z_t, t} > 
		\boldsymbol{\mathbf{x}}_{t}^{\top} \boldsymbol{\beta}_{\pi^*_{t}, Z_t, t} \mid \xi^{(1)}_t, \xi^{(2)}_t\right)\mathbb{P}\left( \xi^{(2)}_t \right) \\
		& + \
		\mathbb{P}\left(\boldsymbol{\mathbf{x}}_{t}^{\top} \breve{\boldsymbol{\beta}}_{\pi^*_{t}, Z_t, t} > 
		\boldsymbol{\mathbf{x}}_{t}^{\top} \boldsymbol{\beta}_{\pi^*_{t}, Z_t, t}\mid \overline{\xi^{(1)}_t}, \xi^{(2)}_t\right)\mathbb{P}\left(\overline{\xi^{(2)}_t}\right)\\
		\leq&\  \mathbb{P}\left(\boldsymbol{\mathbf{x}}_{t}^{\top} \breve{\boldsymbol{\beta}}_{\pi^*_{t}, Z_t, t} > 
		\boldsymbol{\mathbf{x}}_{t}^{\top} \boldsymbol{\beta}_{\pi^*_{t}, Z_t, t} \mid \xi^{(1)}_t, \xi^{(2)}_t\right) +
		\mathbb{P}\left(\overline{\xi^{(2)}_t}\right).
	\end{aligned}
	$$
	By combining \eqref{eqn:6-1} and \eqref{eqn:6-2} and applying Lemma \ref{optimistic}, we have the conclusion
	$$
	\mathbb{P}\left(\pi_{t} \notin \mathcal{S}(t) \mid \xi^{(1)}_t, \xi^{(2)}_t \right) \geq \frac{1}{2 \sqrt{2\pi e}}-\frac{1}{t^{2}}.
	$$
\end{proof}

\section{Useful Lemmas}\label{appendix_lemma_1}

\begin{lemma}\label{variance_sum_1}
	For fixed $j$ and $k$, we have the following bound
	$$
	\sigma^2\sum_{t=1: \pi_t=k, Z_t=j} \mathbf{x}_t^\top\widetilde{\mathbf{C}}_{k,j,t} \mathbf{x}_t \leq c_1 d \log \left(1+ c_2 n_j\right),
	$$
	where $c_1 = \frac{ \lambda_d^{-1}  dc_x}{ \log (1+ \sigma^{-2} \lambda_d^{-1}  dc_x)}$, $c_2 = \sigma^{-2}\lambda_d^{-1}c_x$ and  $n_j$ is the total time steps of bandit $j$.
\end{lemma}
\begin{proof}
	For every individual term in the summation, we have
	\begin{align}\label{x_Ctilde_x}
		\notag \mathbf{x}_t^\top\widetilde{\mathbf{C}}_{k,j,t} \mathbf{x}_t
		&\leq \widetilde{c}_1 \log\left(1+ \mathbf{x}_t^\top\widetilde{\mathbf{C}}_{k,j,t} \mathbf{x}_t\right) \\
		&= \widetilde{c}_1 \log \det\left(\mathbf{I}_d + \widetilde{\mathbf{C}}_{k,j,t}^{\frac{1}{2}} \mathbf{x}_t\mathbf{x}_t^\top\widetilde{\mathbf{C}}_{k,j,t}^{\frac{1}{2}}\right),
	\end{align}
	where the inequality is from the fact: for $x\in[0,u]$,
	$$
	x = \frac{x}{\log(1+x)}\log(1+x) \leq \frac{u}{\log(1+u)}\log(1+x).
	$$
	The constant $\widetilde{c}_1$ in (\ref{x_Ctilde_x}) is derived from eigenvalue bounds of covariance matrix,
	$$
	\begin{aligned}
		\mathbf{x}_t^\top\widetilde{\mathbf{C}}_{k,j,t} \mathbf{x}_t
		\leq \lambda_{\max}(\widetilde{\mathbf{C}}_{k,j,t}) dc_x 
		\leq \lambda_{\min}^{-1}(\sigma^2\boldsymbol{\Sigma}_k^{-1}) dc_x \leq \sigma^{-2} \lambda_d^{-1}  dc_x,
	\end{aligned}
	$$
	following that $\widetilde{c}_1$ is
	$$
	\widetilde{c}_1 = \frac{\sigma^{-2} \lambda_d^{-1}  dc_x}{ \log (1+ \sigma^{-2} \lambda_d^{-1}  dc_x)}.
	$$
	
	For the log-determinant term in (\ref{x_Ctilde_x}), the matrix determinant property follows 
	$$
	\begin{aligned}
		\log \det\left(\mathbf{I}_d + \widetilde{\mathbf{C}}_{k,j,t}^{\frac{1}{2}} \mathbf{x}_t\mathbf{x}_t^\top\widetilde{\mathbf{C}}_{k,j,t}^{\frac{1}{2}}\right) 
		=&\ \log \det\left(\widetilde{\mathbf{C}}_{k,j,t}^{\frac{1}{2}}\left( \widetilde{\mathbf{C}}_{k,j,t}^{-1}+ \mathbf{x}_t\mathbf{x}_t^\top\right)\widetilde{\mathbf{C}}_{k,j,t}^{\frac{1}{2}} \right)\\
		=&\ \log \det\left( \widetilde{\mathbf{C}}_{k,j,t}^{-1}+ \mathbf{x}_t\mathbf{x}_t^\top \right) - \log \det\left( \widetilde{\mathbf{C}}_{k,j,t}^{-1}\right).
	\end{aligned}
	$$
	Then we have
	$$
	\begin{aligned}
		\sum_{t=1: \pi_t=k, Z_t=j} \log \det\left(\mathbf{I}_d + \widetilde{\mathbf{C}}_{k,j,t}^{\frac{1}{2}} \mathbf{x}_t\mathbf{x}_t^\top\widetilde{\mathbf{C}}_{k,j,t}^{\frac{1}{2}}\right)
		=&\ \log \det\left( \widetilde{\mathbf{C}}_{k,j,n+1}^{-1} \right) - \log \det\left( \widetilde{\mathbf{C}}_{k,j,1}^{-1}\right) \\
		=&\ \log \det\left( \sigma^{-2}\boldsymbol{\Sigma}_{k}^{\frac{1}{2}}\widetilde{\mathbf{C}}_{k,j,n+1}^{-1} \boldsymbol{\Sigma}_{k}^{\frac{1}{2}}\right)\\
		\leq&\ d\log \frac{1}{\sigma^{2}d}\text{trace}(\boldsymbol{\Sigma}_{k}^{\frac{1}{2}}
		\widetilde{\mathbf{C}}_{k,j,n+1}^{-1}
		\boldsymbol{\Sigma}_{k}^{\frac{1}{2}}), 
	\end{aligned}
	$$
	where the first inequality is from the trace-determinant inequality. 
	Noting $\widetilde{\mathbf{C}}_{k,j,t}^{-1} = \mathbf{X}_{k,j,t}^\top\mathbf{X}_{k,j,t} + \sigma^{2}\boldsymbol{\Sigma}_k^{-1}$, we have 
	$$
	\begin{aligned}
		\text{trace}(\boldsymbol{\Sigma}_{k}^{\frac{1}{2}}
		\widetilde{\mathbf{C}}_{k,j,n+1}^{-1}
		\boldsymbol{\Sigma}_{k}^{\frac{1}{2}}) 
		=&\ \text{trace}(\boldsymbol{\Sigma}_{k}^{\frac{1}{2}}
		\left(\mathbf{X}_{k,j,n+1}^\top\mathbf{X}_{k,j,n+1} + \sigma^{2}\boldsymbol{\Sigma}_k^{-1} \right) 
		\boldsymbol{\Sigma}_{k}^{\frac{1}{2}})\\
		=&\ \sigma^{2}d + \sum_{t=1: \pi_t=k, Z_t=j}^n \mathbf{x}_t^\top\boldsymbol{\Sigma}_{k}\mathbf{x}_t \\
		\leq&\ \sigma^{2}d +  \lambda_d^{-1}dc_x n_j.
	\end{aligned}
	$$
	Thus, we have the bound
	$$
	\sum_{t=1: \pi_t=k, Z_t=j} \log \det\left(\mathbf{I}_d + \widetilde{\mathbf{C}}_{k,j,t}^{\frac{1}{2}} \mathbf{x}_t\mathbf{x}_t^\top\widetilde{\mathbf{C}}_{k,j,t}^{\frac{1}{2}}\right) \leq  d\log \left(1+ c_2 n_j\right).
	$$
	where $c_2 = \sigma^{-2}\lambda_d^{-1}c_x$. Letting $c_1 = \sigma^2 \widetilde{c}_1$ completes the proof.
\end{proof}

\begin{lemma}\label{variance_sum_2}
	For fixed $k$, 
	we have following bound: 
	\begin{align*}
		\sum_{t=1: \pi_t=k} \sigma^4\mathbf{x}_t^\top \widetilde{\mathbf{C}}_{k,Z_t,t} \boldsymbol{\Sigma}_k^{-1} \boldsymbol{\Phi}_{k0,t}	\boldsymbol{\Sigma}_k^{-1}\widetilde{\mathbf{C}}_{k,Z_t,t} \mathbf{x}_t 
		\leq c_3 d \log \left(1+ c_4 N\right),
	\end{align*}
	where $c_3 = c_5 c_6,\ c_5 = \frac{\lambda_d^{-2} \lambda_1^{2} \lambda^{-1} dc_x}{ \log (1+ \sigma^{-2} \lambda_d^{-2} \lambda_1^{2} \lambda^{-1} dc_x)},\ c_6 =  1 + \sigma^{-2} \lambda_d^{-1}  dc_x$, and $c_4=\lambda_1 \lambda^{-1}$.
\end{lemma}
\begin{proof}
	The proof of this lemma follows the same structure as that of Lemma \ref{variance_sum_1}, differing only in a few technical details. First, each term in the summation can be bounded by
	$$
	\begin{aligned}
		\sigma^2\mathbf{x}_t^\top \widetilde{\mathbf{C}}_{k,Z_t,t} \boldsymbol{\Sigma}_k^{-1} \boldsymbol{\Phi}_{k0,t}
		\boldsymbol{\Sigma}_k^{-1}\widetilde{\mathbf{C}}_{k,Z_t,t}\mathbf{x}_t 
		\leq&\  \widetilde{c}_5 \log\left(1+ \sigma^2\mathbf{x}_t^\top \widetilde{\mathbf{C}}_{k,Z_t,t} \boldsymbol{\Sigma}_k^{-1} \boldsymbol{\Phi}_{k0,t}
		\boldsymbol{\Sigma}_k^{-1}\widetilde{\mathbf{C}}_{k,Z_t,t}\mathbf{x}_t\right) ,
	\end{aligned}
	$$
	where the constant $\widetilde{c}_5 = \frac{\sigma^{-2} \lambda_d^{-2} \lambda_1^{2} \lambda^{-1} dc_x}{ \log (1+ \sigma^{-2} \lambda_d^{-2} \lambda_1^{2} \lambda^{-1} dc_x)}$ is derived as follows: 
	$$
	\begin{aligned}
		\sigma_k^2\mathbf{x}_t^\top \widetilde{\mathbf{C}}_{k,Z_t,t} \boldsymbol{\Sigma}_k^{-1} \boldsymbol{\Phi}_{k0,t}
		\boldsymbol{\Sigma}_k^{-1}\widetilde{\mathbf{C}}_{k,Z_t,t} \mathbf{x}_t 
		\leq&\ \sigma^2\lambda^2_{\max}(\widetilde{\mathbf{C}}_{k,Z_t,t}) \lambda^2_{\max}(\boldsymbol{\Sigma}_k^{-1} ) \lambda_{\max}(\boldsymbol{\Phi}_{k0,t}) dc_x \\
		\leq&\ \sigma^{-2} \lambda_d^{-2} \lambda_1^{2} \lambda^{-1} dc_x.
	\end{aligned}
	$$
	Applying the inequality that $\log(1+x) \leq c\log(1+x/c)$ for any $x \geq 0$ and constant $c \geq 1$, and letting $c_6 = 1 + \sigma^{-2} \lambda_d^{-1}  dc_x > 1$, we obtain
	$$
	\begin{aligned}
		&\log\left(1+ \sigma^2\mathbf{x}_t^\top \widetilde{\mathbf{C}}_{k,Z_t,t} \boldsymbol{\Sigma}_k^{-1} \boldsymbol{\Phi}_{k0,t}
		\boldsymbol{\Sigma}_k^{-1}\widetilde{\mathbf{C}}_{k,Z_t,t}\mathbf{x}_t\right)  \\
		\leq &\  c_6 \log\left(1+ \sigma^2\mathbf{x}_t^\top \widetilde{\mathbf{C}}_{k,Z_t,t} \boldsymbol{\Sigma}_k^{-1} \boldsymbol{\Phi}_{k0,t}		\boldsymbol{\Sigma}_k^{-1}\widetilde{\mathbf{C}}_{k,Z_t,t}\mathbf{x}_t/c_6 \right) \\	
		\leq &\ c_6\log \det\left(\mathbf{I}_d + \sigma^2\boldsymbol{\Phi}_{k0,t}^{\frac{1}{2}}\boldsymbol{\Sigma}_k^{-1}\widetilde{\mathbf{C}}_{k,Z_t,t} \mathbf{x}_t\mathbf{x}_t^\top \widetilde{\mathbf{C}}_{k,Z_t,t}	\boldsymbol{\Sigma}_k^{-1}
		\boldsymbol{\Phi}_{k0,t}^{\frac{1}{2}}/c_6\right)\\
		= &\  c_6\left[\log \det\left( \boldsymbol{\Phi}_{k0,t}^{-1} +
		\sigma^2\boldsymbol{\Sigma}_k^{-1}\widetilde{\mathbf{C}}_{k,Z_t,t} \mathbf{x}_t\mathbf{x}_t^\top \widetilde{\mathbf{C}}_{k,Z_t,t}	\boldsymbol{\Sigma}_k^{-1}/c_6\right) 
		- \log \det\left(\boldsymbol{\Phi}_{k0,t}^{-1}\right)\right],
	\end{aligned}
	$$
	where the second inequality is from the matrix determinant property. Using the result of Lemma \ref{Phi_difference}, we have
	$$
	\begin{aligned}
		\log\left(1+\sigma^2 \mathbf{x}_t^\top \widetilde{\mathbf{C}}_{k,Z_t,t} \boldsymbol{\Sigma}_k^{-1} \boldsymbol{\Phi}_{k0,t}	\boldsymbol{\Sigma}_k^{-1}\widetilde{\mathbf{C}}_{k,Z_t,t}\mathbf{x}_t\right)
		\leq&\ c_6 \left[\log \det\left( \boldsymbol{\Phi}_{k0,s}^{-1} \right) - 
		\log \det\left(\boldsymbol{\Phi}_{k0,t}^{-1}\right) \right],
	\end{aligned}
	$$
	where $s$ is the next time step choosing arm $k$. 
	Summing this term over all time steps $t$ in which arm $k$ is chosen, we have
	$$
	\begin{aligned}
		\sum_{t=1: \pi_t=k} \log\left(1+ \sigma^2\mathbf{x}_t^\top \widetilde{\mathbf{C}}_{k,Z_t,t} \boldsymbol{\Sigma}_k^{-1} \boldsymbol{\Phi}_{k0,t}
		\boldsymbol{\Sigma}_k^{-1}\widetilde{\mathbf{C}}_{k,Z_t,t}\mathbf{x}_t\right)  
		\leq&\ c_6 \left[\log \det\left( \boldsymbol{\Phi}_{k0,n+1}^{-1} \right) - 
		\log \det\left(\lambda\mathbf{I}_d\right) \right]\\
		\leq&\ c_6 d \log \frac{1}{d}\text{trace}\left( \lambda^{-1}\boldsymbol{\Phi}_{k0,n+1}^{-1} \right)\\
		\leq&\ c_6 d\log \left(1+\lambda_1 \lambda^{-1}N\right),
	\end{aligned}
	$$
	where the second inequality follows from the trace-determinant inequality, and the third inequality follows from the conclusion of Lemma \ref{XVX}, specifically,
	$$
	\begin{aligned}
		\text{trace}\left(\boldsymbol{\Phi}_{k0,n+1}^{-1} \right) &= \sum_{j=1}^N \text{trace}\left(\mathbf{X}_{k,j,n+1}^\top \mathbf{V}_{k,j,n+1}^{-1} \mathbf{X}_{k,j,n+1} \right) + \lambda d \\
		&\leq  \sum_{j=1}^N \text{trace}\left(\boldsymbol{\Sigma}_{k}^{-1}\right) + \lambda d \leq \lambda_1dN + \lambda d. 	
	\end{aligned}
	$$
	Let $c_5 = \sigma^2 \widetilde{c}_5$ and $c_4=\lambda_1 \lambda^{-1}$. This completes the proof.
\end{proof}

\begin{lemma}\label{Phi_difference}
	Let $s$ and $t$ denote two consecutive time steps at which the same arm $k$ is selected, with $s > t$. Then, we have
	$$
	\sigma^2\boldsymbol{\Sigma}_k^{-1}\widetilde{\mathbf{C}}_{k,Z_t,t} \mathbf{x}_t\mathbf{x}_t^\top \widetilde{\mathbf{C}}_{k,Z_t,t}	\boldsymbol{\Sigma}_k^{-1}/c  \leq \boldsymbol{\Phi}_{k0,s}^{-1} - \boldsymbol{\Phi}_{k0,t}^{-1},
	$$
	where $c =  1 + \sigma^{-2} \lambda_d^{-1}  dc_x$.
\end{lemma} 
\begin{proof}
	Since the matrix $\boldsymbol{\Phi}_{k0,t}^{-1}$ is updated only when arm $k$ is selected, most terms in $\boldsymbol{\Phi}_{k0,s}^{-1}$ and $\boldsymbol{\Phi}_{k0,t}^{-1}$ remain identical. This leads to substantial cancellation when taking their difference. 
	Specifically, at time step $t$, we handle with $Z_t$-th bandit, and if arm $k$ is chosen, only the term $\mathbf{X}_{k,Z_t,t}^\top\mathbf{V}_{k,Z_t,t}^{-1}\mathbf{X}_{k,Z_t,t}$ in $\boldsymbol{\Phi}_{k0,t}^{-1}$ is updated. Consequently, their difference simplifies to the following expression:
	$$
	\begin{aligned}
		\boldsymbol{\Phi}_{k0,s}^{-1} - \boldsymbol{\Phi}_{k0,t}^{-1} 
		=&\ \mathbf{X}_{k,Z_t,s}^\top\mathbf{V}_{k,Z_t,s}^{-1}\mathbf{X}_{k,Z_t,s} -  \mathbf{X}_{k,Z_t,t}^\top\mathbf{V}_{k,Z_t,t}^{-1}\mathbf{X}_{k,Z_t,t}\\
		=&\ \sigma^2 \boldsymbol{\Sigma}_{k}^{-1} \widetilde{\mathbf{C}}_{k,Z_t,t} \boldsymbol{\Sigma}_{k}^{-1} - \sigma^2 \boldsymbol{\Sigma}_{k}^{-1} \widetilde{\mathbf{C}}_{k,Z_t,s} \boldsymbol{\Sigma}_{k}^{-1}\\
		=&\ \sigma^2 \boldsymbol{\Sigma}_{k}^{-1}
		\left[ \widetilde{\mathbf{C}}_{k,Z_t,t} - \left(\widetilde{\mathbf{C}}_{k,Z_t,t}^{-1} + \mathbf{x}_t\mathbf{x}_t^\top\right)^{-1} \right] 
		\boldsymbol{\Sigma}_{k}^{-1} \\
		=&\ \sigma^2 \boldsymbol{\Sigma}_{k}^{-1}\widetilde{\mathbf{C}}_{k,Z_t,t}^{\frac{1}{2}}\left[\mathbf{I}_d - \left(\mathbf{I}_d + \widetilde{\mathbf{C}}_{k,Z_t,t}^{\frac{1}{2}}\mathbf{x}_t\mathbf{x}_t^\top\widetilde{\mathbf{C}}_{k,Z_t,t}^{\frac{1}{2}} \right)^{-1} \right] \widetilde{\mathbf{C}}_{k,Z_t,t}^{\frac{1}{2}}\boldsymbol{\Sigma}_{k}^{-1}.
	\end{aligned}
	$$
	where the second equality is based on the result of Lemma \ref{XVX}. Letting $\mathbf{v} = \widetilde{\mathbf{C}}_{k,Z_t,t}^{\frac{1}{2}}\mathbf{x}_t$,  we have
	$$
	\begin{aligned}
		\boldsymbol{\Phi}_{k0,s}^{-1} - \boldsymbol{\Phi}_{k0,t}^{-1} 
		&= \sigma^2 \boldsymbol{\Sigma}_{k}^{-1}\widetilde{\mathbf{C}}_{k,Z_t,t}^{\frac{1}{2}}\left[\mathbf{I}_d - \left(\mathbf{I}_d + \mathbf{v}\mathbf{v}^\top\right)^{-1} \right] \widetilde{\mathbf{C}}_{k,Z_t,t}^{\frac{1}{2}}\boldsymbol{\Sigma}_{k}^{-1}.
	\end{aligned}
	$$
	Note that $\left(\mathbf{I}_d + \mathbf{v}\mathbf{v}^\top\right)^{-1} = \mathbf{I}_d - \mathbf{v}\left(1+ \mathbf{v}^\top\mathbf{v}\right) \mathbf{v}^\top$, and that $1+ \mathbf{v}^\top\mathbf{v}$ admits the following upper bound:
	$$
	1 + \mathbf{v}^\top\mathbf{v} = 1 + \mathbf{x}_t^\top\widetilde{\mathbf{C}}_{k,j,t} \mathbf{x}_t
	\leq 1 + \sigma^{-2} \lambda_d^{-1}  dc_x.
	$$
	We then derive the following lower bound for  $\boldsymbol{\Phi}_{k0,s}^{-1} - \boldsymbol{\Phi}_{k0,t}^{-1}$, as follows
	$$
	\boldsymbol{\Phi}_{k0,s}^{-1} - \boldsymbol{\Phi}_{k0,t}^{-1} \geq \frac{\sigma^{2}}{1 + \sigma^{-2} \lambda_d^{-1}  dc_x} \boldsymbol{\Sigma}_{k}^{-1}\widetilde{\mathbf{C}}_{k,Z_t,t}
	\mathbf{x}_t\mathbf{x}_t^\top \widetilde{\mathbf{C}}_{k,Z_t,t}\boldsymbol{\Sigma}_{k}^{-1}.
	$$
	Let $c =  1 + \sigma^{-2} \lambda_d^{-1}  dc_x $. This completes the proof.
\end{proof}

\section{Auxiliary Lemmas}\label{appendix_lemma_2}
\begin{lemma}\citep{abbasi2011improved}\label{abbasi_inequality}
	Let $\{ \mathscr{F}_t \}_{t=0}^{\infty}$ be a filtration. Let $\{ \eta_t \}_{t=1}^{\infty}$ be a real-valued stochastic process such that $\eta_t$ is $\mathscr{F}_t$-measurable and $\eta_t$ is conditionally $R$-sub-Gaussian for some $R \geq 0$.
	Let $\{ \mathbf{x}_t \}_{t=1}^{\infty}$ be an $\mathbb{R}^d$-valued stochastic process such that $\mathbf{x}_t$ is $\mathscr{F}_{t-1}$-measurable and $\mathbf{x}_t$ is bounded in Euclidean norm, i.e., $\|\mathbf{x}_t\|_2  \leq L$ for some constant $L > 0$ and all $t \geq 1$. Assume that $\mathbf{V}$ is a $d \times d$ positive definite matrix with eigenvalues $\lambda_1, \lambda_2, \ldots, \lambda_d$ sorted in descending order: $\lambda_1 \geq \lambda_2 \geq \cdots \geq \lambda_d > 0$. For any $t \geq 0$, define
	$$
	\overline{\mathbf{V}}_t = \mathbf{V} + \sum_{s=1}^t \mathbf{x}_s \mathbf{x}_s^{\top} \quad \text{and} \quad \mathbf{S}_t = \sum_{s=1}^t \eta_s \mathbf{x}_s.
	$$
	Then, for any $\delta > 0$, with probability at least $1 - \delta$, for $d \geq 2$ and all $t \geq 0$,
	$$
	\| \mathbf{S}_t \|_{\overline{\mathbf{V}}_t^{-1}}^2 \leq  R^2 d \log \left( \frac{\lambda_1+tL^2/d}{\lambda_{d}\delta} \right).
	$$
\end{lemma}

\begin{lemma}[Woodbury Matrix Identity]\label{woodbury-matrix}
	For matrices $\mathbf{A}$, $\mathbf{U}$, $\mathbf{C}$ and $\mathbf{V}$ of appropriate dimensions which $\mathbf{A}$ and $\mathbf{C}$ are invertible, the inverse of the matrix sum $\mathbf{A} + \mathbf{U}\mathbf{C}\mathbf{V}$ can be computed as
	\begin{equation*}
		(\mathbf{A} + \mathbf{U}\mathbf{C}\mathbf{V})^{-1} = 
		\mathbf{A}^{-1} - \mathbf{A}^{-1}\mathbf{U}(\mathbf{C}^{-1} + \mathbf{V}\mathbf{A}^{-1}\mathbf{U})^{-1}\mathbf{V}\mathbf{A}^{-1}.
	\end{equation*}
\end{lemma}

\begin{lemma}[Matrix Determinant Property]\label{matrix-determinant}
	For any $\mathbf{x} \in \mathbb{R}^d$, we have
	$$
	\det(1 + \mathbf{x}^\top\mathbf{x}) = \det(\mathbf{I}_d + \mathbf{x}\mathbf{x}^\top).
	$$
\end{lemma}

\begin{lemma}[Trace-determinant Inequality]\label{trace-determinant}
	For any positive definite matrix $\mathbf{A} \in \mathbb{R}^{d \times d}$, we have
	$$
	\log\det(\mathbf{A}) \leq d\log \frac{1}{d}\text{trace}(\mathbf{A}).
	$$
\end{lemma}

\begin{lemma}[Anti-concentration Inequality]\label{concentration_gaussian}
	For a standard Gaussian distributed random variable  $Z$  with mean  $0$  and variance  $1$, it holds that for any  $x$  that
	$$
	\frac{1}{ \sqrt{2\pi} }\frac{x}{x^2+1} \exp \left(-x^{2} / 2\right) \leq \mathbb{P}(Z > x ) \leq  \exp \left(-x^{2} / 2\right).
	$$
\end{lemma}

	\section{Regret Performance under the Oracle Policy}
	\label{sec:oracle}
	\cite{xu2025multitask} defines regret based on the oracle policy.
	Specifically, at each time step $t$, when $Z_t = j \in[N]$, we define an optimal policy $\pi^*_{j,t}$ that knows the true arm parameters $\{\boldsymbol{\beta}_{k,j}\}_{k\in[K]}$ of bandit $j$ in advance and always selects the arm with the highest expected reward. 
	That is,  
	\begin{equation*}
		\pi^*_{j,t} = \arg\max_{k\in[K]} \mathbf{x}_t^\top \boldsymbol{\beta}_{k,j},
	\end{equation*}
	where $\mathbf{x}_t$ is the context vector observed at time $t$. 
	The expected regret at time $t$ for bandit instance $j$ is defined as the difference between the expected reward of the optimal arm chosen by $\pi^*_{j,t}$ and that of the arm chosen by our policy $\pi_{j,t}$: 
	$r_{j,t} = \mathbb{E}\left[ \left(\mathbf{x}_t^\top \boldsymbol{\beta}_{\pi^*_{j,t}, j} - \mathbf{x}_t^\top \boldsymbol{\beta}_{\pi_{j,t}, j}\right) \mid Z_t = j \right]$. 
	Given that $\mathbb{P}(Z_t = j) = p_j$, the law of total expectation implies that the overall expected regret at time $t$ can be written as $r_{t} = \sum_{j=1}^N p_j r_{j,t}$. 
	The cumulative regret over $n$ time steps is then $R_{n} = \sum_{t=1}^n  r_{t}$. 
	and the instance-specific cumulative regret for each bandit instance $j$, given by $R_{j,n} = \sum_{t=1}^n  p_j r_{j,t}$.
	We evaluate the performance of the cumulative regret in a simulation study, and the results are reported in Figure  \ref{fig: synthetic regret curve plot, allkj}, under the same model setting as in Figure \ref{fig:synthetic_regret_curve_plot}. 
	From Figure  \ref{fig: synthetic regret curve plot, allkj}, the performance is very similar to that shown in Figure \ref{fig:synthetic_regret_curve_plot}.

	\begin{figure}[!ht]
		\centering
		\subfigure[Cumulative regret in the data-balanced setting]{
			\includegraphics*[width=0.45\textwidth]{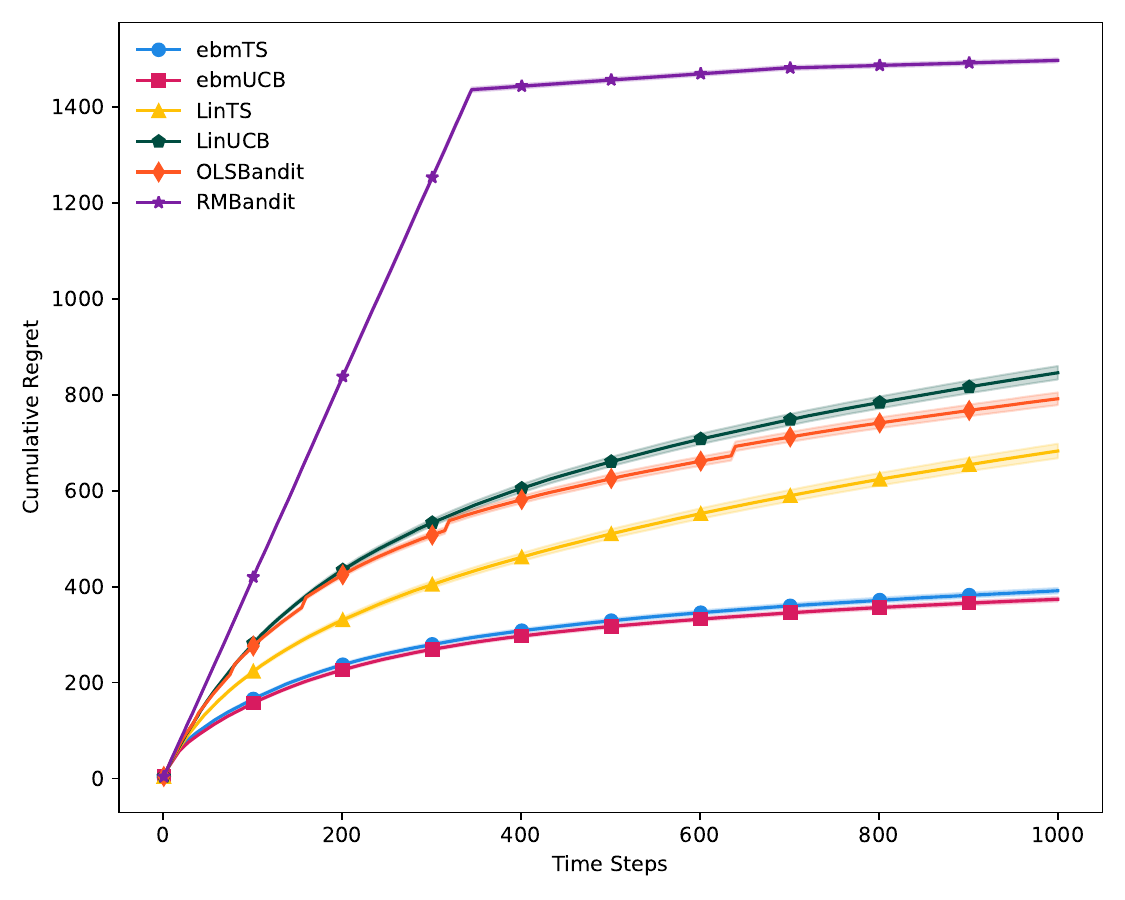}
			\label{fig: data-balanced, allkj}
		}
		\subfigure[Cumulative regret in the data-poor setting]{
			\includegraphics*[width=0.45\textwidth]{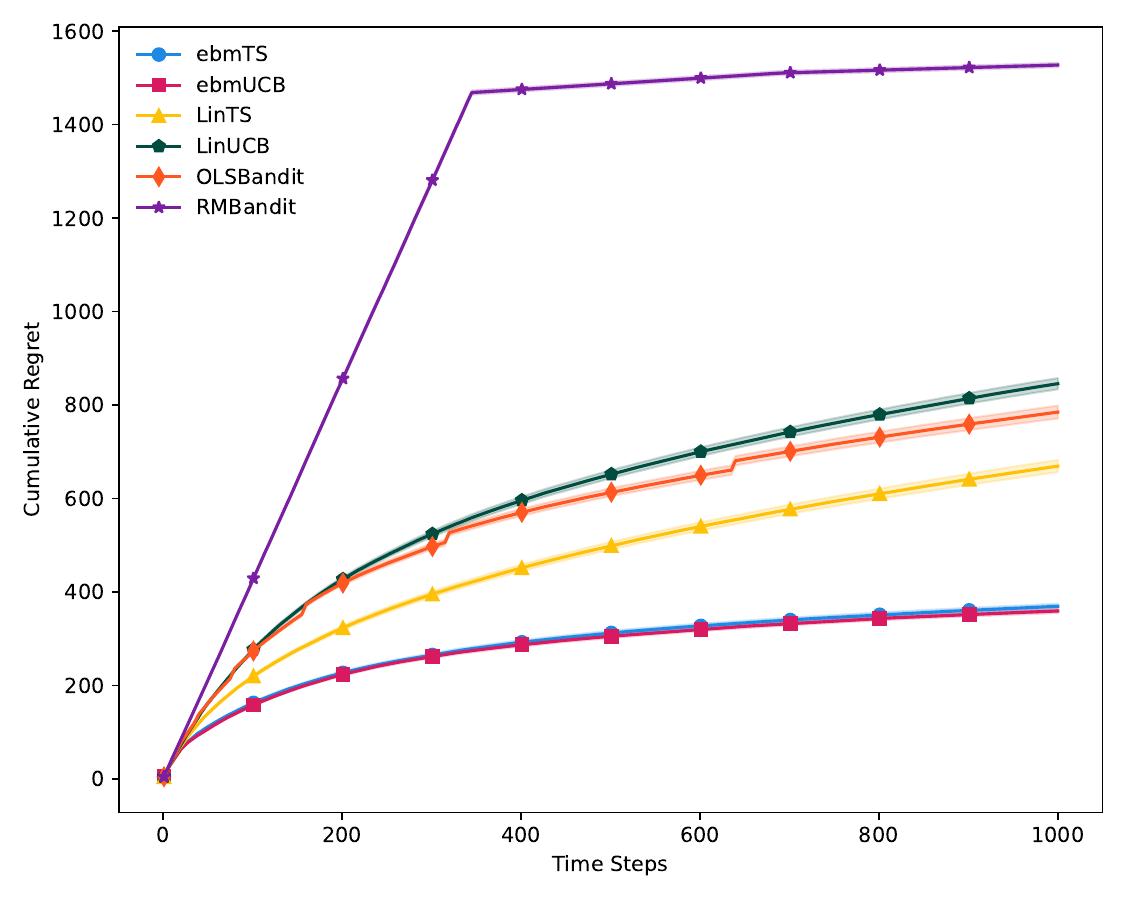}
			\label{fig: data-poor, allkj}		
		}
		\caption{Performance under $N=10$, $K=5$, $d=3$.}
		\label{fig: synthetic regret curve plot, allkj}
	\end{figure}
	
	\section{Sensitivity to Hyperparameters}
	\label{appendix:hyperparams}
	
	We examine the sensitivity of \textsc{ebmTS} and \textsc{ebmUCB} to the hyperparameters $a$ and $\lambda$. We vary $a \in \{0.05, 0.1, 0.2\}$ and $\lambda \in \{0.0005, 0.001, 0.01\}$ on a grid and report the cumulative regret at the final time step as a function of each hyperparameter, while fixing the other at its default value. As shown in Figure~\ref{fig:hyparams}, we rank the candidate configurations by cumulative regret in ascending order. Both algorithms remain robust across the considered hyperparameter values and consistently outperform the baseline.
		\begin{figure}[!ht]
		\centering
		\subfigure[Hyperparameters of \textsc{ebmTS}.]{
			\includegraphics*[width=0.8\textwidth]{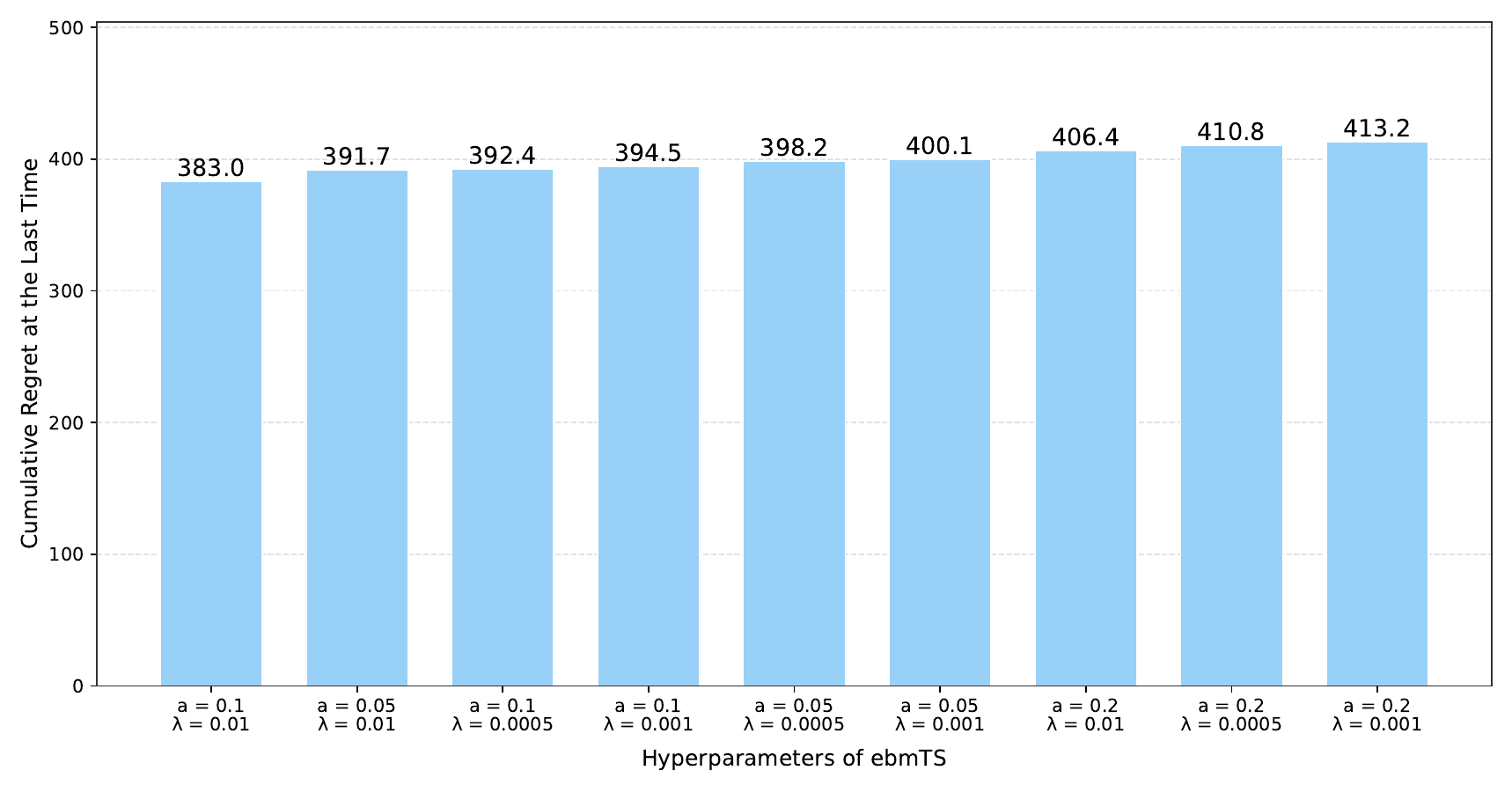}
			\label{fig:hyparams_ebmTS}
		}\\
		\subfigure[Hyperparameters of \textsc{ebmUCB}.]{
			\includegraphics*[width=0.8\textwidth]{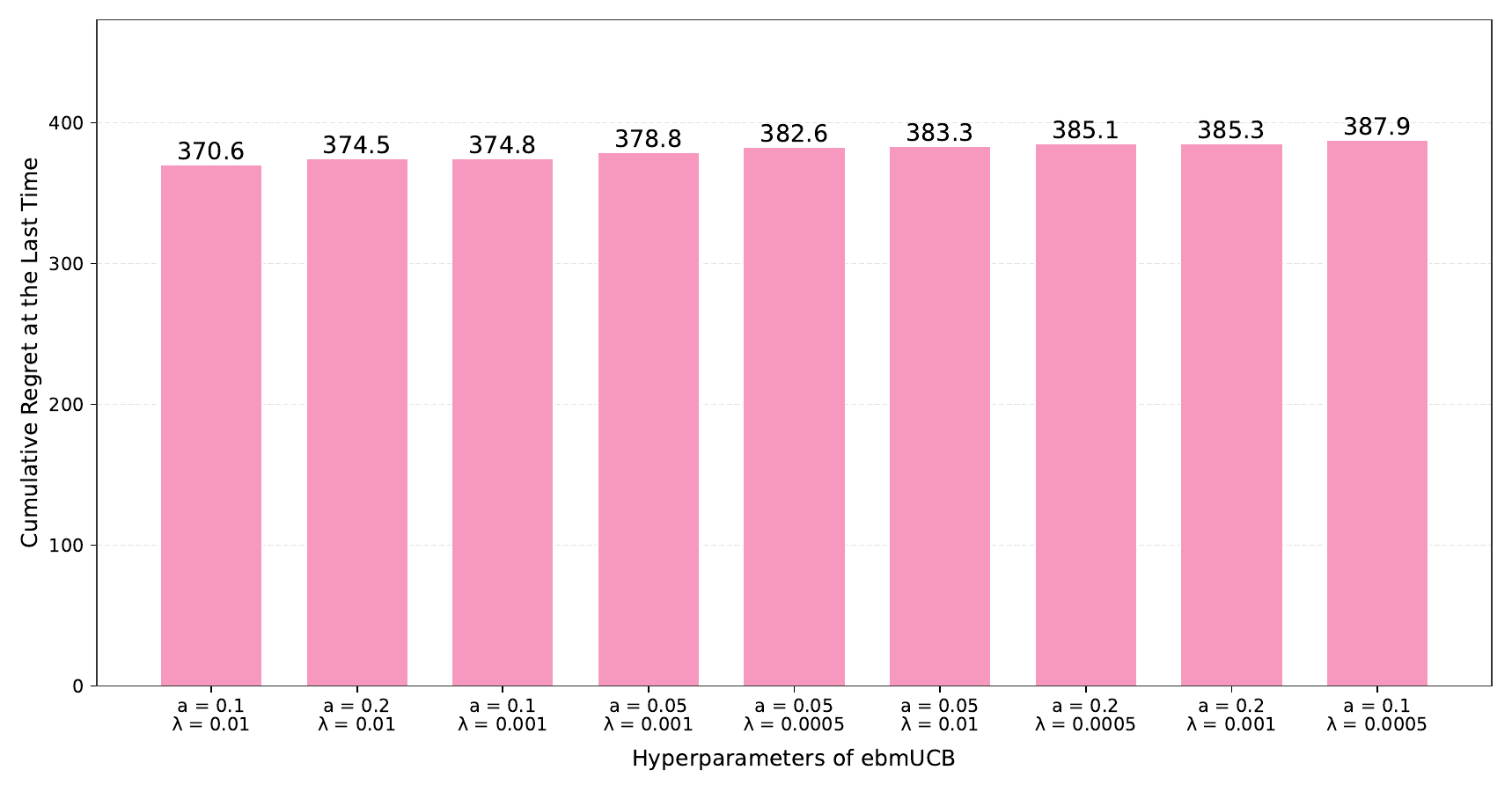}
			\label{fig:hyparams_ebmUCB}		
		}
		\caption{Performance under $N=10$, $K=5$, $d=3$.}
		\label{fig:hyparams}
	\end{figure}

\end{document}